\documentclass{article}

\usepackage{microtype}
\usepackage{graphicx}
\usepackage{subfigure}
\usepackage{booktabs} 
\usepackage{enumitem}

\usepackage{hyperref}

\usepackage[accepted]{icml2020}

\icmltitlerunning{StochasticRank: Globally Convergent Stochastic Ranking}

\usepackage{amsmath,amsfonts,bbold}

\DeclareMathOperator*{\argmin}{arg\,min}
\usepackage{amsthm}
\newtheorem{statement}{Statement}
\newtheorem{definition}{Definition}

\newtheorem{lemma}{Lemma}
\newtheorem{corollary}{Corollary}
\newtheorem{theorem}{Theorem}

\newcommand{\R}{\mathbb{R}}

\begin{document}

\twocolumn[
\icmltitle{StochasticRank: Global Optimization of Scale-Free Discrete Functions}



\begin{icmlauthorlist}
\icmlauthor{Aleksei Ustimenko}{ya}
\icmlauthor{Liudmila Prokhorenkova}{ya,mipt,hse}
\end{icmlauthorlist}

\icmlsetsymbol{equal}{*}
\icmlaffiliation{ya}{Yandex, Moscow, Russia}
\icmlaffiliation{mipt}{Moscow Institute of Physics and Technology, Dolgoprudny, Moscow Region, Russia}
\icmlaffiliation{hse}{Higher School of Economics, Moscow, Russia}

\icmlcorrespondingauthor{Aleksei Ustimenko}{austimenko@yandex-team.ru}
\icmlkeywords{learning to rank, information retrieval, non-convex optimization, stochastic differential equations}

\vskip 0.3in
]



\printAffiliationsAndNotice{}  

\begin{abstract}
In this paper, we introduce a powerful and efficient framework for direct optimization of ranking metrics. The problem is ill-posed due to the discrete structure of the loss, and to deal with that, we introduce two important techniques: stochastic smoothing and novel gradient estimate based on partial integration.
We show that classic smoothing approaches may introduce bias and present a universal solution for a proper debiasing.
Importantly, we can guarantee \textit{global convergence} of our method by adopting a recently proposed Stochastic Gradient Langevin Boosting algorithm. Our algorithm is implemented as a part of the CatBoost gradient boosting library and outperforms the existing approaches on several learning-to-rank datasets. In addition to ranking metrics, our framework applies to any scale-free discrete loss function. 
\end{abstract}

\section{Introduction}

The quality of ranking algorithms is traditionally measured by ranking quality metrics such as Normalized Discounted Cumulative Gain (NDCG), Expected Reciprocal Rank (ERR), Mean Average Precision (MAP), Mean Reciprocal Rank (MRR), and so on~\cite{Sakai13}. These metrics are defined on a list of documents sorted by their predicted relevance to a query and capture the utility of that list for users of a search engine, who are more likely to scan documents starting at the top. 
Direct optimization of ranking metrics is an extremely challenging problem since sorting makes them piecewise constant (as functions of predicted relevances), so they are neither convex nor smooth.
Many algorithms were proposed for different ranking objectives in the learning-to-rank (LTR) research field. We refer to~\citet{Liu:2009} for a systematic overview of some classic methods. 

To deal with the discrete structure of a ranking loss, one can use some smooth approximation, which is easier to optimize. This technique lies behind such well-known algorithms as 
SoftRank~\cite{Taylor2008}, ApproxNDCG~\cite{qin2010general}, RankNet~\cite{burges2010ranknet}, etc.
The obtained smooth function can be optimized by gradient-based methods and, in particular, by Stochastic Gradient Boosting (SGB) that is known to be the learning algorithm behind most state-of-the-art LTR frameworks and is commonly preferred by major search engines~\cite{Chapelle2010,Yin2016}. 
Unfortunately, all known smoothing approaches suffer from \textit{bias} (see Sections~\ref{sec:bias_example}-\ref{sec:pi_resolution}) which prevents them from truly direct optimization. Moreover, smoothed ranking loss functions are non-convex, and existing algorithms can guarantee only local optima.

Our ultimate goal is to solve these problems and propose a \textit{truly direct} LTR algorithm with provable guarantees of global convergence and generalization. 
We adopt a theoretical approach, so we start with formal definitions of the class of ranking losses and its generalization to \textit{scale-free (SF) discrete loss functions} (Section~\ref{sec:general_def}). Our results hold for the general class of SF losses, which, in addition to all ranking metrics, includes, e.g., a recently proposed loss function for Learning-to-Select-with-Order~\cite{vorobev2019learning}.
Then, to mitigate the discontinuity of the loss, we use stochastic smoothing. We prove that previous smoothing-based approaches are inconsistent with the underlying loss (due to the problem of ties, which we discuss in the next section) and propose a universal solution to this problem (relevance-based consistent smoothing, see Section~\ref{sec:pi_resolution}). Next, we derive a novel stochastic gradient estimate, which can be applied to the entire class of SF losses (see Section~\ref{sec:ccs}). The obtained estimate has low variance and uniformly bounded error, which is crucial for our analysis. 
Finally, to guarantee \textit{global} convergence of the algorithm, we adopt a recently proposed Stochastic Gradient Langevin Boosting (SGLB) algorithm~\cite{SGLB}. SGLB is based on a well studied Stochastic Gradient Langevin Dynamics~\cite{GelfandGAA, DBLP:journals/corr/RaginskyRT17, Erdogdu:2018:GNO:3327546.3327636} and converges globally for a wide range of loss functions including non-convex ones.
We adapt SGLB to our setting and obtain a gradient boosting algorithm that converges globally for the entire class of SF loss functions with provable generalization guarantees (see Section~\ref{sec:global}). 

To sum up, to the best of our knowledge, the proposed StochasticRank algorithm is the first \textit{globally converging} LTR method with provable guarantees that optimizes \textit{exactly} the underlying ranking quality loss. 
Stochastic\-Rank is implemented within the official CatBoost library~\cite{catboost, CatBoost_documentation}. Our experiments show that StochsticRank outperforms the existing approaches on several LTR datasets. 

The rest of the paper is organized as follows. In the next section, we briefly overview the related research on learning to rank. In Section~\ref{formalization}, we formalize the problem and, in particular, define a general class of ranking loss functions. In Section~\ref{sec:smoothing}, we formulate the problem of smoothing bias and propose an unbiased solution. Then, in Section~\ref{sec:ccs}, we derive a novel stochastic gradient estimate for the whole class of loss functions under consideration. In Section~\ref{sec:global}, we show how SGLB can be used to achieve global convergence. Finally, Section~\ref{Experiments} empirically compares the proposed algorithm with existing approaches, and Section~\ref{sec:conclusion} concludes the paper.

\section{Related Work}\label{sec:related_work}

Usually, researches divide all LTR methods into three categories: pointwise, pairwise, and listwise~\cite{Liu:2009}.

\textbf{Pointwise} are the earliest and simplest methods: they approximate relevance labels based on simple or ordinal regression or classification. Such methods were shown to be ineffective for LTR, since loss functions they optimize (e.g., RMSE for relevance labels) differ significantly from the target ranking metric, e.g., NDCG@k.

\textbf{Pairwise} methods make a step forward and focus on pairwise preferences and thus known to outperform pointwise approaches significantly. Nevertheless, pairwise approaches still suffer from the problem of solving a different task rather than optimizing a ranking quality objective. 

\textbf{Listwise} methods try to solve the problem directly by developing either smooth proxies of the target ranking metric like SoftRank~\cite{Taylor2008}, BoltzRank~\cite{volkovs2009boltzrank}, ApproxNDCG~\cite{qin2010general}, RankNet~\cite{burges2010ranknet} or by Majorization-Minimization procedure that builds a convex upper bound on the metric on each iteration like LambdaMART~\cite{wu2010adapting}, LambdaLoss~\cite{lambdaloss}, PermuRank~\cite{xudirect}, SVMRank~\cite{cao2006adapting}, etc.

As discussed in the previous section, algorithms based on smooth approximations suffer from bias and local optima. 
Also, there are listwise approaches that try to optimize the target loss function without smoothing. For instance, Direct\-Rank~\cite{tan2013direct} constructs an ensemble of decision trees, where the values in the leaves are chosen to optimize the original loss. However, due to greediness, this approach can guarantee only local optima. 

Finally, let us note that algorithms optimizing a convex upper bound instead of the original loss cannot be truly direct since the optimum for the upper bound can potentially be far away from the true optimum. This is nicely illustrated by~\citet{nguyen2013algorithms} for accuracy optimization.
Let us also mention a recent approach for improving learning-to-rank algorithms by adding Gumbel noise to model predictions~\cite{48689}. This is a regularization technique since it builds a convex upper bound on any given convex loss (e.g., LambdaMART).\footnote{\citet{Nesterov2017RandomGM} prove this for Gaussian noise, but the same result generalizes to any centered noise.}
Thus, from a theoretical point of view, this approach cannot be truly direct since it uses convex upper bounding. 

The issue of smoothing bias mentioned in the introduction is connected to the problem of ties: if predicted relevances of some documents coincide, one has to order them somehow to compute a ranking metric. This situation may occur when two documents have equal features. More importantly, ties are always present in boosting algorithms based on discrete weak learners such as decision trees. Unfortunately, this problem is rarely addressed in LTR papers. In practice, it is reasonable to use the \textit{worst} permutation. First, due to strong penalization, it would force an optimization algorithm to avoid ties. Second, in practice, one cannot know how a production system would rank the items, and often some attribute negatively correlated with relevance is used (e.g., sorting by a bid in online auctions). The importance of using the worst permutation is also discussed by~\citet{rudin2018direct}, and this ordering is adopted in some open-source libraries like CatBoost~\cite{catboost}. An alternative choice is to compute the expected value of a ranking metric for a \textit{random} permutation. This choice is rarely used in practice, since it is computationally complex and gives non-trivial scores to trivial constant predictions, but is often assumed (explicitly or implicitly) by LTR algorithms~\cite{kustarev2011smoothing}.

\section{Problem Formalization}
\label{formalization}

\subsection{Examples of Ranking Loss Functions}\label{sec:ranking_examples}

Before we introduce a general class of loss functions, let us define classic ranking quality functions widely used throughout the literature and in practice.\footnote{To obtain the \textit{loss} function from the corresponding \textit{quality} function, we multiply it by $-1$. }
These loss functions depend on $z$, which is a vector of \textit{scores} produced by the model, and $r$, which is a vector of relevance labels for a given query. The length of these vectors is denoted by $n$ and can be different for different queries.

Let $s = \mathrm{argsort}(z)$, i.e., $s_i$ is the index of a document at $i$-th position if documents are ordered according to their scores (if $z_{i} = z_{j}$ for $j \ne i$, then we place the less relevant one first).
Let us define $\mathrm{DCG@k}$, where $k$ denotes the number of top documents we are interested in:
\begin{equation}
    \mathrm{DCG@k}(z, r) = \sum_{i=1}^{\min\{n, k\}} \frac{2^{r_{s_i}} - 1}{2^4\log_2 (i + 1)}, 
\end{equation}
where $r_i\in [0, 4]$ are relevance labels. 
This quality function is called \emph{Discounted Cumulative Gain}: for each document, the numerator corresponds to \textit{gain} for the relevance, while the denominator \textit{discounts} for a lower position. 
$\mathrm{NDCG@k}$ is a normalized variant of $\mathrm{DCG@k}$:

\begin{equation}
    \mathrm{NDCG@k}(z, r) = \frac{\mathrm{DCG@k}(z, r)}{\max_{z' \in \mathbb{R}^n}{\mathrm{DCG@k}}(z', r)}.
\end{equation}
{\it Expected Reciprocal Rank} $\mathrm{ERR@k}$ assumes that $r_j \in [0, 1]$:
\begin{equation}
    \mathrm{ERR@k}(z, r) = \sum_{i=1}^{\min\{n, k\}} \frac{r_{s_i}}{i}\prod_{j=1}^{i-1} (1-r_{s_j}). 
\end{equation}
\textit{Mean reciprocal rank} ($\mathrm{MRR}$) is used for binary relevance labels $r_j \in \{0, 1\}$:
\begin{equation}
    \mathrm{MRR}(z, r) = \sum_{i=1}^{n} \frac{r_{s_i}}{i}\prod_{j=1}^{i-1} (1-r_{s_j}),
\end{equation}
which is the inverse rank of the first relevant document.

Finally, let us define a quality function for the LSO (learning to select with order) problem introduced by \citet{vorobev2019learning}, which is not exactly a ranking metric, but has a similar structure. 
The order of elements is predefined (documents are sorted by their indices), but the list of documents to be included is determined by $(\mathbb{1}_{\{z_i > 0\}})_{i=1}^n\in \{0, 1\}^n$:
\begin{equation}
    \mathrm{DCG\text{-}RR}(z, r) = \sum_{i=1}^{n} \frac{r_{i}\,\mathbb{1}_{\{z_{i} > 0\}}}{1+\sum_{j < i} \mathbb{1}_{\{z_{j} > 0\}}} \,.
\end{equation}
In the sum above, for each included document we divide its relevance by its rank.

\subsection{Generalized Ranking Loss Functions}\label{sec:general_def}

To develop a stochastic ranking theory, we first formalize the class of loss functions to which our results apply. We start with a very general class of \textit{scale-free (SF) discrete loss functions}. Further, by $\xi$ we denote a vector of context, which may include relevance and any other factors affecting the ranking quality value (like query type or document topic).

\begin{definition} 
\label{GRQF}
A function $L(z, \xi): \coprod_{n > 0} \mathbb{R}^n \times \Xi_n \rightarrow \mathbb{R}$ is a \emph{Scale-Free Discrete Loss Function} iff the following conditions hold:
\begin{itemize}
\item \textbf{Uniform boundedness}: There exists a constant $l > 0$ such that $|L(z, \xi)| \le l$ holds  $\forall n$, $\forall \xi\in \Xi_n$, $\forall z\in\mathbb{R}^n$;
\item \textbf{Discreteness on subspaces}: For each $n\in \mathbb{N}$ and linear subspace $V\subset \mathbb{R}^n$ there exist convex open subsets $U_1, \ldots, U_k\subset V, \, k=k(n, V)$ (w.r.t.~induced topology on $V$), 
mutually disjoint $U_i \cap U_j = \emptyset$ for $i\ne j$, with everywhere dense union $\overline{\cup_i U_i} = V$ ($\overline{X}$ denotes the closure of $X$ w.r.t.~the ambient topology), such that for any $\xi \in \Xi_n$ and $i \le k$ holds $L(z, \xi)\big|_{U_i} \equiv \mathrm{const}(i, \xi, V)$;
\item \textbf{Jumps regularity}: By reusing $U_i$ defined above, for any $z \not\in \cup_i U_i$ either of the following conditions holds: 
$$\lim\inf_{z'\rightarrow z} L(z', \xi) < L(z, \xi)\le \lim\sup_{z'\rightarrow z} L(z', \xi),$$
$$\lim\inf_{z'\rightarrow z} L(z', \xi) = L(z, \xi) = \lim\sup_{z'\rightarrow z} L(z', \xi),$$
where $z' \rightarrow z$ means $z' \in \cup U_i$, $z' \rightarrow z$.
\item \textbf{Scalar freeness}: 
For any $n > 0, \xi \in \Xi_n, z \in \mathbb{R}^n, \lambda > 0$ holds $L(\lambda z, \xi) = L(z, \xi)$.
\end{itemize}
\end{definition}

We denote the class of all SF discrete loss functions by $\mathcal{R}_0$. Informally speaking, $\mathcal{R}_0$ is a class of bounded discrete functions on a sphere.
The jumps regularity property is needed to exclude the breaking points from $\argmin L$. One can show that all loss functions defined in Section~\ref{sec:ranking_examples}, including the LSO loss DCG-RR, belong to $\mathcal{R}_0$. 

StochasticRank out-of-box can be applied to any SF discrete loss function. However, to guarantee global convergence, we need to use consistent smoothing (see Section~\ref{sec:pi_resolution}), which has to be chosen based on the properties of a particular metric. We propose smoothing which is consistent for the whole class of \textit{ranking loss functions} defined below. 

Assume that $\Xi_n = \mathbb{R}^n \times \Xi_n'$ and $\xi \in \Xi_n$ is a tuple $(r, \xi')$, where $r \in \mathbb{R}^n$ is a vector of relevance labels. 
As discussed in Section~\ref{sec:related_work}, a particular definition of a ranking loss depends on tie resolution. When some documents have equal scores, we may either use the worst permutation (as commonly done in practice) or compute the average over all orderings of such documents (as usually assumed by LTR algorithms). The definition below assumes the worst permutation.

\begin{definition}
\label{RQF}
A function $L(z, \xi)\in \mathcal{R}_0$ is a \emph{Ranking Loss Function} iff the following properties hold:
\begin{itemize}
    \item \textbf{Relevance monotonicity}: For each $n > 0$ and  $z, r\in \mathbb{R}^n$, there exists $\epsilon_0 = \epsilon_0(r, z) > 0$ such that $\forall \epsilon \in (0, \epsilon_0]$ $\exists\delta = \delta(\epsilon, r, z) > 0$ such that $\forall z': \|z' - z\| < \delta$:
    $$\lim \sup_{z''\rightarrow z} L(z'', \xi) = L(z' - \epsilon r, \xi).$$
    Informally, $-r$ is the worst direction for the loss function, i.e., near a breaking point with $z_i = z_j$ and $r_i > r_j$ for some $i,j$, it is better to have $z_i > z_j$.
    \item \textbf{Strong upper semi-continuity (s.u.s.c.)}: For each $n > 0$ and  $z, r\in \mathbb{R}^n$:
    $$\lim \sup_{z' \rightarrow z} L(z', \xi) = L(z, \xi).$$
    Informally, this means that if we do not know how to rank two items (i.e., $z_i = z_j$ for $i\ne j$), then we shall rank them by placing the less relevant one first.
    \item \textbf{Translation invariance}:\footnote{This property is assumed only to be consistent with the learning-to-rank literature and can be omitted.}  For any $n > 0, r,z \in \mathbb{R}^n$, $\lambda \in \mathbb{R}$ holds:
    $L(z + \lambda \mathbb{1}_n, \xi) \equiv L(z, \xi)$,
    where $\mathbb{1}_n := (1, \ldots, 1) \in \mathbb{R}^n$.
    \item \textbf{Pairwise decision boundary:}\footnote{This condition can also be removed, but it simplifies the analysis of smoothing bias.} Partition of the space for discreteness on subspaces $\{U_i\}$ for $\mathbb{R}^n$ can be obtained as connected components of $\mathbb{R}^n\backslash\cup_{i,j}\{z:z_i-z_j=0\}$, similarly for an arbitrary subspace $V$.
\end{itemize}
\end{definition}
We denote this class of functions by $\mathcal{R}_1$. 
It can be shown that $\mathcal{R}_1$ includes all ranking losses defined in Section~\ref{sec:ranking_examples}, but not the LSO loss $\mathrm{DCG\text{-}RR}$ which does not satisfy Relevance monotonicity.

Let us now define a class $\mathcal{R}_1^{soft}$, where instead of the worst ranking for ties, we consider the expected loss of a random ranking. For this, we replace the s.u.s.c. condition by:

\begin{itemize}
    \item \textbf{Soft semi-continuity (s.s.c.)}: For each $n > 0$ and $r,z \in \mathbb{R}^n$ we have:
    $$\lim_{\sigma \rightarrow 0_+} \mathbb{E}L(z + \sigma \varepsilon, \xi) = L(z, \xi),$$
    where $\varepsilon \sim \mathcal{N}(\mathbb{0}_n, I_n)$ is a normally distributed random variable.
\end{itemize}

We will show that under some restrictive conditions (that are commonly assumed in the LTR literature), it does not matter which of the two definitions we use ($\mathcal{R}_1$ or $\mathcal{R}_1^{soft}$) as they coincide almost surely and have equal $\argmin L$ sets. However, we will explain why these conditions do not hold in practice
and in general the minimizers for $\mathcal{R}_1$ and $\mathcal{R}_1^{soft}$ do not coincide.

\subsection{Model Assumptions}

We assume that for each $n > 0$ and $\xi \in \Xi_n$ there is a model $f_\xi(\theta):\mathbb{R}^m\rightarrow \mathbb{R}^n$ such that $f_\xi(\theta) = \Phi_\xi \theta$ for some matrix $\Phi_\xi \in \mathbb{R}^{n\times m}$, where $\theta \in \mathbb{R}^{m}$ is a vector of parameters (independent from $\xi$) and $m \in \mathbb{N}$ is the number of parameters. Typically, each row of $\Phi_\xi$ is a feature vector.
Gradient boosting over decision trees satisfies this assumption. Indeed, let us consider all possible trees of a fixed depth formed by a finite number of binary splits obtained by binarization of the initial feature vectors. To get a linear model, we say that $\theta$ is a vector of leaf weights of these trees and $\Phi_\xi$ is a binary matrix formed by the binarized feature vectors.

We will also assume that $\langle \mathbb{1}_{n}, z\rangle_2 = 0$. Indeed, instead of $z = f_{\xi}(\theta)$ we can define the model as $z = f_{\xi}(\theta) - \frac{1}{n}\mathbb{1}_{n}^Tf_{\xi}(\theta)\mathbb{1}_{n}$, which is equivalent due to the translation invariance property. 

\subsection{Data Distribution}
Assume that we are given some distribution $\xi\sim \mathcal{D}$ on $\Xi := \coprod_{n > 0} \Xi_n$ meaning that $\xi$ also implicitly incorporates information about the number of items $n$, i.e., for $\xi\in \Xi$ there exists a unique number $n > 0$ so that $\xi \in \Xi_n$. 
$\mathcal{D}$ is some unknown distribution, e.g., the distribution of queries submitted to a search system. We are given a finite i.i.d. sample $\xi_1, \ldots, \xi_N \sim \mathcal{D}$ that corresponds to the train set. Let $\mathcal{D}_{N} := \frac{1}{N}\sum_{i=1}^N \delta_{\xi_i}$ be the empirical distribution. 

\subsection{Optimization Target}

The assumptions and definitions above allow us to define the expected (generalized) ranking quality for the function $L\in \mathcal{R}_0$ with respect to $\xi \sim \mathcal{D}$ and model parameters $\theta \in \mathbb{R}^m$:
$\mathcal{L}(\theta) := \mathbb{E}_{\xi\sim \mathcal{D}} L(f_\xi (\theta), \xi).
$
Our ultimate goal is to 
find $\argmin_{\theta} \mathcal{L}(\theta)$.
However, since the distribution $\mathcal{D}$ is unknown, we have only i.i.d.~samples $\xi_1, \ldots, \xi_N$ as defined above. So, we consider the expected ranking quality under the empirical distribution $\mathcal{D}_{N}$:
\begin{equation*}
    \mathcal{L}_{N}(\theta) := \mathbb{E}_{\xi\sim \mathcal{D}_{N}} L(f_\xi (\theta), \xi) = \frac{1}{N}\sum_{i=1}^{N} L(f_{\xi_i}(\theta), \xi_i).
\end{equation*}

We want to optimize $\mathcal{L}(\theta)$ \emph{globally} by optimizing $\mathcal{L}_{N}(\theta)$. 
This is possible because of the stability of global minimizers even for discontinuous functions: for $N \gg 1$ an almost minimizer of $\mathcal{L}_N(\theta)$ should be an almost minimizer of $\mathcal{L}(\theta)$~\cite{artstein}. 

Thus, we need to find a global minimizer of $\mathcal{L}_N$. 
Due to the discrete structure, we can ignore sets of zero Lebesgue measure. Recall that essential infimum ($\mathrm{ess }\inf$) is infimum that ignores sets of zero measure and $\mathrm{int}\, U$ denotes an open interior of the set $U$.

\begin{definition}
\label{ArgminDef}
For any function $\mathcal{L}(\theta):\mathbb{R}^m\rightarrow \mathbb{R}$ with $\mathcal{L}_* := \mathrm{ess }\inf_{\theta\in\mathbb{R}^m} \mathcal{L}(\theta) > -\infty$, we define
$$\argmin \mathcal{L}(\theta) := \mathrm{int}\big\{\theta\in\mathbb{R}^m: \mathcal{L}(\theta) = \mathcal{L}_* \big\}.$$
\end{definition}
We need this unusual definition because of the discrete structure of our loss: we want to exclude the breaking points from $\argmin$. One can see that  despite $L(\cdot, \cdot)$ satisfies Jumps regularity, the function $\mathcal{L}_N(\theta)$ does not have to. 

\begin{statement}\label{statement-1}
The set $\argmin_{\theta\in\mathbb{R}^m} \mathcal{L}_N(\theta)$ is not empty.
\end{statement}
The proof is straightforward (see Appendix~\ref{appendix:statement-1}).

\section{Stochastic smoothing}\label{sec:smoothing}

\subsection{Smoothing of Scores}

The discrete structure of ranking loss functions prevents their effective optimization. Hence, some smoothing is needed and a natural approach for this is mollification~\cite{ErmolievNorkin, 10.2307/1999442}, i.e., adding randomness to parameters. We refer to Appendix~\ref{appendix:mollification} for the formal definition and the reasons why this approach is not applicable in our case.

Thus, instead of acting on the level of parameters $\theta$, we act on the level of scores $z$: $L_{\xi}^{\pi}(z, \sigma):=\mathbb{E}L(z+\sigma\varepsilon, \xi)$, where $\varepsilon$ is a random variable with p.d.f. $\pi(z)$.
We multiply the noise by $\sigma$ to preserve Scalar-freeness in a sense that $L^{\pi}_\xi(\lambda z, \lambda \sigma) = L_\xi^{\pi}(z, \sigma)$ for any $\lambda > 0$.

In the linear case $f(\theta) = \Phi\,\theta$, if $\mathrm{rk}\,\Phi = n$, it is not hard to show the convergence of minimizers. 
However, in general, we cannot assume $\mathrm{rk}\,\Phi = n$. 
In particular, this property is violated in the presence of ties that \textit{always} occur in gradient boosting due to the discrete nature of decision trees. As a result, there is a \emph{smoothing bias} that alters the set of minimizers. 

\subsection{Simple Example of Smoothing Bias}\label{sec:bias_example}

Within this section, assume for simplicity that we are dealing with one function $L(z) := L(z, \xi):\mathbb{R}^n\rightarrow \mathbb{R}$ for some arbitrary fixed $n$ and $\xi \in \Xi_n$. Let $\Phi = \Phi_\xi\in \mathbb{R}^{m\times n}$ and $\mathcal{L}(\theta) := L(\Phi\,\theta)$. 
To clearly see how a \emph{smoothing bias} can be introduced, consider the case when $\mathrm{im}(\Phi) \subset \mathbb{R}^n \,\backslash\, {\cup_{i=1}^k U_i}$, where $U_i$ are from the Discreteness on subspaces assumption for $V = \mathbb{R}^n$. Denote by $c_1, \ldots, c_k\in \mathbb{R}$ the values of $L(z)$ on the corresponding subsets $U_i$. Consider the functions $\mathcal{L}(\theta)$ and 
$\mathcal{L}^{\pi}(\theta) := \lim_{\sigma\rightarrow 0_+} \mathbb{E}_{\varepsilon\sim \pi}L(\Phi\theta + \sigma \varepsilon)$.

The value of $\mathcal{L}^{\pi}(\theta)$ is fully determined by $\pi$, $c_1, \ldots, c_k$ and the subsets $U_1, \ldots, U_k$ in the following way: $\mathcal{L}^{\pi}(\theta) = \sum_i \alpha_i c_i$ with $$\alpha_i = \alpha_i(\pi, \theta, U_1, \ldots, U_k) = \lim_{\sigma\rightarrow 0_+} \mathbb{P}(\Phi\,\theta + \sigma \varepsilon \in U_i).$$ 

In contrast, the value $\mathcal{L}(\theta)$ depends on the values $c_1, \ldots, c_k$ much weaker: for fixed $\theta$, consider the values $c_1', \ldots, c_{k'}'$ that correspond to $U_i$ such that $\Phi\theta \in \overline{U_i}$, then the only limitation we have is $\min c_{i}' < \mathcal{L}(\theta)\le \max c_{i}'$ (this is required by Jumps regularity), which clearly allows more flexibility than the linear combination defined above.

In LTR, the issue of smoothing bias is connected to the problems of ties: the situations when $z_i = z_j$ and $r_i\ne r_j$. 

\subsection{Consistent Smoothing}\label{sec:pi_resolution}

\begin{definition}
\label{pi-RQF}
We say that the family of distributions $\pi_{\xi}(z): \coprod_{n > 0} \mathbb{R}^n \times \Xi_n  \rightarrow \mathbb{R}_{+}$ is a \emph{consistent smoothing} for $L(z, \xi) \in \mathcal{R}_0$ and for the model $f_\xi$ iff for each $n > 0$, $\xi \in \Xi_n$ the following limit holds almost surely locally uniform in $\theta$: 

$$L(f_\xi (\theta), \xi) = \lim_{\sigma\rightarrow 0_+} L_{\xi}^{\pi}(f_\xi(\theta),\sigma).$$ 
\end{definition}
If $\pi$ is smooth enough and consistent, then the function $\mathcal{L}^{\pi}_N(\theta, \sigma) := \frac{1}{N}\sum_{i=1}^N L_{\xi_i}^{\pi}(f_{\xi_i}(\theta), \sigma)$ is also smooth and almost surely locally uniformly approximates the discrete loss $\mathcal{L}_N(\theta)$ as $\sigma \rightarrow 0_+$.

To optimize ranking losses, it is important to find a consistent smoothing $\pi$ for functions in $\mathcal{R}_1$.
Fortunately, we can do this with an arbitrary precision by shifting the normal distribution by $-\mu r$ for large enough $\mu$. Relevance monotonicity and s.u.s.c. imply the following pointwise limit: 
\begin{multline*}
\lim_{\mu\rightarrow\infty}\lim_{\sigma\rightarrow 0_+} \mathbb{E}_{\varepsilon \sim \mathcal{N}(-\mu r, I_n)}L(z +\sigma\varepsilon, \xi) \\ 
= \lim_{\mu\rightarrow \infty}\lim_{\sigma\rightarrow 0_+} \mathbb{E}_{\varepsilon \sim \mathcal{N}(\mathbb{0}_n, I_n)}L(z - \sigma\mu r + \sigma \varepsilon, \xi)
= L(z, r)  \,.
\end{multline*}
This can be strengthened to the following theorem, which is proven in Appendix~\ref{appendix:theorem-1}.
\begin{theorem}\label{theorem-1}
\label{pi-resolution-of-R1}
$\pi_{\xi, \mu} = \mathcal{N}(-\mu r, I_n)$ is a consistent smoothing for $\mathcal{R}_1$ as $\mu \rightarrow \infty$. 
Formally, $\forall \theta$ except zero measure $\exists \, \delta > 0$ $\forall \, \epsilon > 0$ $\exists \, \mu > 0$ $\exists\, \sigma_0 > 0$ such that $\forall \sigma 
\in (0, \sigma_0)$ and $\forall\theta':\|\theta - \theta'\| < \delta$ holds $|L_\xi^{\pi}(f_\xi(\theta'), \sigma) - L(f_\xi(\theta'), \xi)| < \epsilon$.
\end{theorem}

By similar arguments, one can show that $\mathcal{N}(\mathbb{0}, I_n)$ is a consistent smoothing for $\mathcal{R}_1^{soft}$. 
Note that in both cases the consistent smoothing is \textit{universal} for the entire class ($\mathcal{R}_1$ of $\mathcal{R}_1^{soft}$), i.e., it is independent from the choice of $f_\xi$.

Thus, LTR problems require non-trivial smoothing to preserve consistency. However, under some restrictive assumptions on the loss and on the model, any smoothing $\pi$ is consistent. 

Recall that $\mathcal{L}_N(\theta) = \frac{1}{N}\sum_{i=1}^N L(\Phi_{\xi_i} \theta, \xi_i)$ and assume that $L(z, \xi) \in \mathcal{R}_0$. 
The following theorem is proven in Appendix~\ref{appendix:theorem-2}.

\begin{theorem}\label{theorem-2}
\label{sufficient-cond-rk}
Consider open and convex subsets $U'_{ij}:=U_{ij}\cap \mathrm{im}\,\Phi_{\xi_i}$. 
If $\forall i$ $\exists j$ s.t. $U'_{ij} \ne \emptyset$ and $\overline{\cup_j U'_{ij}} = \mathrm{im}\,\Phi_{\xi_i}$, then any smoothing $\pi$ 
is consistent for $\mathcal{L}_N(\theta)$.
\end{theorem}

In early literature on LTR, all authors used such conditions implicitly by assuming that scores for all items are different. In contrast, we do not use this assumption as it never holds in practice (e.g., when two documents have equal features). As a result, all existing LTR approaches suffer from a smoothing bias. In contrast, for the LSO problem, any smoothing is consistent, as we discuss in Appendix~\ref{appendix:smoothing-lso}.

\subsection{Scale-Free Acceleration}

It is intuitively clear that for a scale-free function it is better to have a scalar-free approximation. However, for each $\lambda > 0$ we have 
$L_\xi^{\pi}(\lambda z, \sigma) = L_\xi^{\pi}(z, \lambda^{-1}\sigma)$, i.e., the smoothed function is no longer scale-free. 
To enforce scale-freeness, we take a vector $z'$ with $\|z'\|_2 > 0$ and define
$${L}_\xi^{\pi}(z, \sigma|z') 
:= L^{\pi}_\xi\left(z, \frac{\|z\|_2}{\|z'\|_2}\sigma\right).$$
We refer to such smoothing as \emph{Scale-Free Acceleration} (SFA). The obtained function is indeed  scale-free: ${L}_\xi^{\pi}(\lambda z, \sigma|z') \equiv {L}_\xi^{\pi}(z, \sigma|z')$ for any $\lambda > 0$.

Let $\widehat{\sigma}(z) := \frac{\|z\|_2}{\|z'\|_2}\sigma$. In our optimization, we will be interested only in the case when $z' = z_t$ is the vector of scores obtained on $t$-th iteration of the optimization algorithm. 
So, we have $\widehat{\sigma}(z_t) = \sigma$ and SFA \textit{does not} change the scale $\sigma$.

One can imagine a sphere of radius $R = \|z'\|_2$, where we restrict $L^{\pi}_\xi(z, \sigma)$ and homogenize it along the rays from the origin to infinity to obtain a scalar-free function.

\subsection{Smoothing Properties}

Finally, let us discuss regularity assumptions for smoothing on which our optimization method relies. Consider a family of distributions with p.d.f.~$\pi_{\xi}(z)$ with $\xi \in \Xi_n$ for some $n > 0$, $z \in \mathbb{R}^n$. 
We require the following properties:
\begin{itemize}
    \item \textbf{Continuous differentiability:} $\pi_{\xi}(z)$ is $C^{(1)}(\mathbb{R}^n)$, i.e., is differentiable with a continuous derivative.
    \item \textbf{Uniformly bounded derivative:} $\forall n \in \mathbb{N}$, $\forall\xi \in \Xi_n$ we have $\|\nabla_z \pi_{\xi}\|_2 = \mathcal{O}(1)$  uniformly in $z \in \mathbb{R}^n$. 
    \item \textbf{Derivative decay:} $\forall n \in \mathbb{N}$ we have $\|\nabla_z \pi_{\xi}\|_2 = \mathcal{O}(\|z\|_2^{-n-2})$ as $\|z\|_2 \rightarrow \infty$. 
    \item \textbf{Tractable conditional expectations:} conditional densities $\pi_{\xi}^j (z_j) := \pi_{\xi}(z_j|z_{\backslash j})$ are easy to compute.\footnote{We do not use the log-derivative trick, so we do not care about the ability to compute $\frac{\mathrm{d}}{\mathrm{d}z_j}\pi_{\xi}^j (z_j)$ and $\frac{\mathrm{d}}{\mathrm{d}z_j}\log\pi_{\xi}^j (z_j)$, our gradient estimates require only computation of $\pi_{\xi}^j (z_j)$.}
\end{itemize}
 
Clearly, $\mathcal{N}(-\mu r, I_N)$ satisfies these assumptions $\forall \mu \ge 0$. 

\section{Coordinate Conditional Sampling}\label{sec:ccs}

\subsection{Gradient Estimate}\label{sec:gradient-estimation}

In the previous section, we required  the ability to easily compute $\pi_{\xi}^j (z_j) = \pi_{\xi}(z_j|z_{\backslash j})$. This property allows us to do the following trick: we decompose $\pi_{\xi}(z)=\pi_{\xi}^j (z_j)\pi_{\xi}^{\backslash j}(z_{\backslash j})$ with $\pi_{\xi}^{\backslash j}(z_{\backslash j})$ being the marginal distribution for $z_{\backslash j}$. Then, we can represent $L^{\pi}_\xi(z, \sigma)=L_\xi^{\pi} \ast \pi_\xi^{j} \ast \pi_\xi^{\backslash j}$. Note that the convolution is an associative operation that commutes with differentiation and, henceforth,
$$\frac{\partial}{\partial z_j} L_\xi^{\pi}(z, \sigma) = \left(\frac{\partial}{\partial z_j}L_\xi^{\pi}\ast \pi_\xi^{j}\right)\ast \pi_\xi^{\backslash j}.$$
Note that we differentiate by $z_j$ the convolution by the same $z_j$. So, if we want to estimate the gradient unbiasedly, we need to sample $\varepsilon_{\backslash j} \sim \pi_{\xi}^{\backslash j}$ and then compute exactly $\Big(\frac{\partial}{\partial z_j}L^{\pi}_\xi\ast \pi_\xi^{j}\Big)\big((z_j, z_{\backslash j} + \sigma\varepsilon_{\backslash j})\big)$. The resulting estimate would be unbiased by construction. The following lemma suggests how to deal with $\frac{\partial}{\partial z_j}L^{\pi}_\xi \ast \pi_\xi^{j}$.

\begin{lemma}
\label{lemma=finite-jumps}
The function $l_j(z_j) := L((z_j, z_{\backslash j}), \xi):\mathbb{R}\rightarrow\mathbb{R}$ for all $z$ except zero measure has at most $k' \le k(n, \mathbb{R}^n)-1$ ($k$ is from the Discreteness on subspaces assumption) breaking points $b_1, \ldots, b_{k'}$ (possibly depending on $z_{\backslash j}$ and $\xi$) and can be represented as:
$$l_j(z_j) = \sum_{s=1}^{k'} \Delta l_j(b_s) \mathbb{1}_{\{z_j \le b_s\}} + \mathrm{const}(z_{\backslash j},\xi),$$
$$\Delta l_j(b_s) := \lim_{\epsilon\rightarrow 0_+} l_j(b_s + \epsilon) - l_j(b_s - \epsilon).$$
\end{lemma}
All results of this section are proven in Appendix~\ref{appendix:CCS}. 

Based on the above lemma, we prove the following theorem.

\begin{theorem}\label{theorem-3}
The derivative $\frac{\partial}{\partial z_j}L^{\pi}_\xi(z, \sigma)$ is equal to:
\[
    -\sigma^{-1} \cdot\mathbb{E}_{\varepsilon_{\backslash j}\sim \pi_\xi^{\backslash j}} \sum_{s=1}^{k'} \Delta l_j(b_s) \pi^{j}_\xi(\sigma^{-1}(b_s - z_j)),
\]
where $k'$ and $b_s = b_s(z_{\backslash j} + \sigma\varepsilon_{\backslash j})$ are from Lemma~\ref{lemma=finite-jumps}.
\end{theorem}

\begin{corollary}\label{cor-1}
For LTR losses, the above formula becomes:
\begin{multline*}
    \frac{\partial}{\partial z_j}L^{\pi}_\xi(z, \sigma) =
    -\sigma^{-1} \cdot \\ \cdot\mathbb{E}_{\varepsilon_{\backslash j}\sim \pi_\xi^{\backslash j}} \sum_{s=1}^{n} \Delta l_j(z_s + \sigma\varepsilon_s) \pi^{j}_\xi(\sigma^{-1}(z_s - z_j) + \varepsilon_s).
\end{multline*}

\end{corollary}

Uniform boundedness of $\Delta l_j$ and $\pi$ implies the following. 
\begin{statement}
The estimate is uniformly bounded by $\mathcal{O}(\sigma^{-1})$.
\end{statement}

Proceeding analogously with each coordinate $j \in \{1, \ldots, n\}$, we obtain an unbiased estimate of $\nabla_z L_\xi^{\pi}(z, \sigma)$ that is uniformly bounded, in contrast to the classic estimate $\sigma^{-1}(L(z+\sigma\varepsilon)-L(z))\varepsilon$~\cite{Nesterov2017RandomGM} obtained by the log-derivative trick for the normal distribution that is also known as REINFORCE~\cite{reinforce}. Uniform boundedness is crucial since without it we would not be able to claim global convergence. We call such estimate \emph{Conditional Coordinate Sampling} (CCS) and denote it by $\widehat{\nabla}_{CC}L^{\pi}_\xi(z, \sigma)$.

Note that for each coordinate when estimating $\widehat{\nabla}_{CC}L_\xi^{\pi}(z, \sigma)$ we use the shared noise vector $\varepsilon \sim \pi_{\xi}$, i.e., the components of the gradient can have non-trivial covariation, but due to the uniform boundness the covariation is also uniformly bounded by $\mathcal{O}(\sigma^{-1})$.

Finally, let us discuss the complexity of computing $\widehat{\nabla}_{CC} L_{\xi}^{\pi}(z, \sigma)$. The following result follows from Appendix~\ref{algos}.
\begin{statement}
The estimate $\widehat{\nabla}_{CC} L_{\xi}^{\pi}(z, \sigma)$ can be computed in:
\begin{itemize}
\item $\mathcal{O}((k+\log n)n)$ operations and $\mathcal{O}(n)$ additional memory for $\mathrm{(N)DCG}@k$ and $\mathrm{ERR}@k$.
\item $\mathcal{O}(n\log n)$ operations and $\mathcal{O}(1)$ memory for $\mathrm{MRR}$.

\end{itemize}
\end{statement}

\subsection{SFA Gradient Estimate}
It is not hard to generalize CCS to SFA. The following theorem holds.

\begin{theorem}\label{theorem-4}
For $\widehat{\sigma}(z) = \Big(\frac{\|z\|_2}{\|z'\|_2}\Big)\sigma$ at $z' = z$ we have:
$$\nabla_z L^{\pi}(z, \widehat{\sigma}(z)) = \nabla_z L^{\pi}_\xi -\Big\langle \nabla_z L_\xi^{\pi},  \frac{z}{\|z\|_2}\Big\rangle_2 \frac{z}{\|z\|_2}.$$
\end{theorem}
\begin{corollary}
Unbiased CCS estimate for SFA can be obtained by orthogonalizing $\widehat{\nabla}_{CC}L_\xi^{\pi}(z, \sigma)$ and~$z$.
\end{corollary}
Since orthogonalization reduces the norm of the estimate, it necessarily reduces the variance, so we obtain the following corollary. 
\begin{corollary}
SFA CCS estimate has a lower variance than the original CCS.
\end{corollary}

The intuition for the orthogonalization is based on Scalar-freenees: the function $L(z, \xi)$ \textit{does not change} along $z$ direction, so this direction in the gradient $\nabla_z L_\xi^{\pi}$ does not contribute to $L(z, \xi)$ optimization. 

As we need to deal with possibility of $z = z' = \mathbb{0}_n$, we introduce a parameter $\nu > 0$ and replace $\|z\|_2$ by $\|z\|_2 + \nu$:
\begin{multline*}
\widehat{\nabla}_{CC} L^{\pi}_\xi(z, \sigma|z', \nu)\big|_{z'=z} := \widehat{\nabla}_{CC} L^{\pi}_\xi(z, {\sigma}) \\-\Big\langle \widehat{\nabla}_{CC} L_\xi^{\pi}(z, {\sigma}),  \frac{z}{\|z\|_2+\nu}\Big\rangle_2 \frac{z}{\|z\|_2+\nu}.
\end{multline*}
\begin{lemma}
Bias of SFA CCS estimate is uniformly bounded:
$$\big|\mathbb{E}\widehat{\nabla}_{CC} L^{\pi}_\xi(z, \sigma_0|z', \nu) - \nabla_z L^{\pi}_\xi(z, \widehat{\sigma})\big| = \mathcal{O}\Big(\frac{1}{\|z\|+\nu}\Big).$$
As a consequence, if $\nu \rightarrow \infty$ or $\|z\| \rightarrow \infty$, then the estimate is asymptotically unbiased. 
\end{lemma}
Thus, for the convergence analysis we consider only $\widehat{\nabla}_{CC} L^{\pi}(z, \sigma)$ since the estimate $\widehat{\nabla}_{CC} L^{\pi}(z, \sigma_0|z', \nu)$ can be made unbiased by varying the parameter $\nu > 0$. In practice, we consider $\widehat{\nabla}_{CC} L^{\pi}_\xi(z, \sigma_0|z', \nu)$ with fixed $\nu = 10^{-2}$ as we observed that this parameter performs well enough. Moreover, SFA can be seen as a bias--variance tradeoff controlled by $\nu > 0$ for CCS estimate of $\nabla_z L_\xi^{\pi}(z, \sigma)$.
For practical comparison of $\widehat{\nabla}_{CC} L^{\pi}_\xi(z, \sigma)$ and $\widehat{\nabla}_{CC} L^{\pi}_\xi(z, \sigma_0|z', \nu)$ we refer to Section~\ref{Experiments}, where we show that SFA gives a significant improvement.

\section{Global Optimization by Diffusion}\label{sec:global}

\subsection{SGLB}

Previously, we discussed the importance of global optimization of $\mathcal{L}_N(\theta)$. As we show in this section, this can be achieved by global optimization of smoothed $\mathcal{L}_N^{\pi}(\theta,\sigma)$ with $\sigma = 1$ (if smoothing is consistent) using the recently proposed Stochastic Gradient Langevin Boosting (SGLB)~\cite{SGLB}. SGLB is easy to apply: essentially, each iteration of standard SGB is modified via model shrinkage and adding Gaussian noise to the gradients. However, the obtained algorithm is backed by strong theoretical results, see~\cite{SGLB} for the details and Appendix~\ref{appendix:sglb} for a brief sketch. 
The global convergence is implied by the fact that as the number of iterations grows, the stationary distribution $p_\beta(F)$ of the predictions $F = (f_{\xi_1}(\theta), \ldots, f_{\xi_N}(\theta))$ concentrates around the global optima of the implicitly regularized loss 
$$
\mathcal{L}^{\pi}_N(F, \sigma, \gamma) = \mathcal{L}^{\pi}_N(F, \sigma) + \frac{\gamma}{2} \|\Gamma F\|_2^2, 
$$ 
where $\Gamma$ is an implicitly defined regularization matrix.
More formally, $p_\beta(F) \propto \exp(-\beta \mathcal{L}^{\pi}_N(F, \sigma, \gamma))$.

Global convergence of SGLB requires Lipschitz smoothness and continuity~\cite{SGLB}. We can ensure this for the entire $\mathcal{R}_0$, which allows us to claim the following theorem (see Appendix~\ref{appendix:theorem-5} for the proof).

\begin{theorem}\label{theorem-5}
SGLB method applied to $\mathcal{L}_N^{\pi}(F, \sigma)$ converges \emph{globally} to optima of $\mathcal{L}_N(F) \equiv \mathcal{L}_N(\theta)$ when used with CCS estimate.
\end{theorem}

The following statement ensures that we can safely fix $\sigma = 1$ and fit only $\gamma$ parameter without loosing any possible solution.
\begin{statement}
$\mathbb{E}_{F\sim p_\beta}\mathcal{L}_N(F) = \mathbb{E}_{F'\sim p_\beta'}\mathcal{L}_N(F')$, where $p_\beta$ corresponds to $(\sigma, \gamma)$ and $p_\beta'$ to $(1, \sigma^2 \gamma)$.
\end{statement}
\begin{proof}
Due to Scalar-freeness, we can write $\mathcal{L}_N^{\pi}(F, \sigma) \equiv \mathcal{L}_N^{\pi}(\sigma^{-1}F, 1)$ and $\frac{\gamma}{2}\|\Gamma F\|_2 \equiv \frac{\sigma^2\gamma}{2}\|\Gamma \big(\sigma^{-1}F\big)\|_2^2.$
Finally, due to Scalar-freeness, the change $F' = \sigma^{-1}F$ does not change the value of $\mathcal{L}_N(F)\equiv \mathcal{L}_N(F')$ and thus the expectation does not change.
\end{proof}

\subsection{Generalization}
\label{Generalization-Section}

\citet{SGLB} related the generalization gap with the uniform spectral gap parameter $\lambda_* \ge 0$ for the distribution $p_\beta(\theta) := \frac{\exp(-\beta \mathcal{L}_N(\theta, \sigma, \gamma))}{\int_{\mathbb{R}^m}\small{\exp(-\beta\mathcal{L}_N(\theta, \sigma, \gamma))}\mathrm{d}\theta}$ (see \citet{DBLP:journals/corr/RaginskyRT17} for the definition of a uniform spectral gap). Here $p_\beta(\theta)$ represents the limiting (as the learning rate goes to zero) distribution of the vector of parameters $\theta$ and is induced by the distribution $p_\beta(F) \propto \exp(-\beta\mathcal{L}(F, \sigma, \gamma))$ using the relation $F = \Phi \theta$. The following theorem is proven in Appendix~\ref{appendix:theorem-6}.

\begin{theorem}
\label{theorem=gen-gap}
The generalization gap $\big|\mathbb{E}_{\theta\sim p_\beta(\theta)} \mathcal{L}^{\pi}(\theta, \sigma) - \mathbb{E}_{\theta\sim p_\beta(\theta)}\mathcal{L}_N^{\pi}(\theta, \sigma)\big|$ can be bounded by:
$$\mathcal{O}\left(\left(\beta+2d+\frac{d^2}{\beta}\right)\frac{\exp(\mathcal{O}(\frac{\beta}{\gamma\sigma^2}))}{\gamma N}\right)\,.$$
\end{theorem}

\section{Experiments}\label{Experiments}

As baseline approaches, we consider the well-known LambdaMART framework optimized for $\mathrm{NDCG@k}$~\cite{wu2010adapting}, NDCG-Loss2++ from the LambdaLoss framework~\cite{lambdaloss}, and SoftRank~\cite{Taylor2008}. We also apply the technique proposed by~\citet{48689} to the baselines, the corresponding methods are called $\mathbb{E}\lambda$-MART and $\mathbb{E}\lambda$-Loss. Similarly to~\citet{lambdaloss}, we set the parameter $\mu$ for NDCG-Loss2++ to be equal to~5. According to our experiments, NDCG-Loss2++ performed significantly better than NDCG-Loss2, which agrees with~\citet{lambdaloss}. 

\subsection{Synthetic Data}
Unfortunately, in practice, we cannot verify if we have reached the global optimum as we cannot evaluate all possible ensembles of trees. But having theoretical guarantees is important as it implies the stability of the algorithm and good generalization. In this section, we describe a simple synthetic test to verify whether StochasticRank can reach the global optimum.

The following dataset is multimodal (has several local optima) for $\mathrm{NDCG}@3$: the number of queries is $N = 2$, first relevance vector is $r_1 = (3, 2, 1)$ and the second is $r_2 = (3, 2)$. We consider the following features for the first query: $x_1 = (1, 0, 0), x_2 = (0, 1, 0), x_3 = (0, 0, 1)$ and for the second $x_3$ and $x_1$ (in the given order). 

We consider this simple synthetic dataset for two reasons: first, it clearly shows that ranking losses are likely to be multimodal; second, it allows us to demonstrate how multimodality prevents existing approaches from reaching the global optimum. 

We limited the tree depth parameter to 3, so one tree can separate all documents with different features. We set the number of iterations to $1000$, learning rate to $0.1$, diffusion temperature to $10^3$, and model-shrink-rate to $10^{-3}$.

The results are shown in Table~\ref{tab:synthetic}. We note that the maximum achievable $\mathrm{NDCG}@3$ for this dataset is $0.917$, i.e., StochasticRank successfully recovers the global optimum while all other approaches converge to a local optimum $0.903$. 

\begin{table}[t]
  \caption{Experimental results on synthetic data.}\label{tab:synthetic}
  \vspace{4pt}
  \centering
  \begin{tabular}{lc}
    \toprule
    Method & $\mathrm{NDCG}@3$\\
    \midrule
      $\lambda$-MART & 0.903 \\
      $\lambda$-Loss & 0.903 \\
      $\mathbb{E}\lambda$-MART & 0.903 \\
      $\mathbb{E}\lambda$-Loss & 0.903\\
      SoftRank       &  0.903 \\
      StochasticRank & \bf{0.917} \\
    \bottomrule
\end{tabular}
\end{table}

\subsection{Real Data}

\paragraph{Datasets} 
For our experiments, we use the following publicly available datasets. First, we use the data from YAHOO! Learning to Rank Challenge~\cite{Chapelle2010}: there are two datasets, each is pre-divided into training, validation, and testing parts. The other datasets are WEB10K and WEB30K released by Microsoft~\cite{QinL13}. Following~\citet{lambdaloss}, we use Fold~1 for these two datasets. 

\paragraph{Quality metrics}
The first metric we use is $\mathrm{NDCG@5}$, which is very common in LTR research.
The second one is $\mathrm{MRR}$, which is a well-known click-based metric. 
Recall that $\mathrm{MRR}$ requires binary labels, so we binarize each label by $\widetilde{y}_i := \mathbb{1}_{\{y_i > 0\}}$.
Notably, while  $\mathrm{MRR}$ is frequently used in online evaluations, it is much less studied compared to $\mathrm{NDCG@k}$ and there are no effective approaches designed for it.
Fortunately, our method can be easily adapted to any ranking metric via a combination of SGLB with Coordinate Conditional Sampling smoothed by Gaussian noise.

\paragraph{Framework} 
We implemented all approaches in CatBoost, which is an open-source gradient boosting library outperforming the most popular alternatives like XGBoost~\cite{Chen:2016} and LightGBM~\cite{LightGBM} for several tasks~\cite{catboost}. LambdaMART can be easily adapted for optimizing $\mathrm{MRR}$, so we implemented both versions. In contrast, LambdaLoss is specifically designed for $\mathrm{NDCG}$ and cannot be easily modified for $\mathrm{MRR}$. For SoftRank we used CCS to estimate gradients, since the original approach is computation and memory demanding, so it is infeasible in gradient boosting which requires~\textit{all} gradients to be estimated at each iteration.

\begin{table}[t]
  \caption{Experimental results.}
  \vspace{4pt}
  \centering
  \begin{tabular}{lccc}
    \toprule
   Method & Dataset & $\mathrm{NDCG@5}$ & $\mathrm{MRR}$ \\
   \midrule
   $\lambda$-MART & Yahoo Set 1 & 74.53 & 90.21 \\
   $\lambda$-Loss & Yahoo Set 1 & 74.73 & - \\
   $\mathbb{E}\lambda$-MART & Yahoo Set 1 & 74.57 & 90.30 \\
   $\mathbb{E}\lambda$-Loss & Yahoo Set 1 & 74.75 & - \\
   SoftRank       & Yahoo Set 1 & 71.98 & 90.17 \\
   \midrule
   SR-$\mathcal{R}_1^{soft}$ & Yahoo Set 1 & 74.68 & \textbf{91.07} \\
   SR-$\mathcal{R}_1$ & Yahoo Set 1 & \textbf{74.92} & 90.97 \\
   \midrule
   \midrule
   $\lambda$-MART & Yahoo Set 2 & 73.87 & 91.48 \\
   $\lambda$-Loss & Yahoo Set 2 & 73.89 & - \\
   $\mathbb{E}\lambda$-MART & Yahoo Set 2 & 73.87 & 91.48 \\
   $\mathbb{E}\lambda$-Loss & Yahoo Set 2 & 73.91 & - \\
   SoftRank       & Yahoo Set 2 & 73.91 &  92.16 \\
   \midrule
   SR-$\mathcal{R}_1^{soft}$ & Yahoo Set 2 & 73.95 & 93.16 \\
   SR-$\mathcal{R}_1$ & Yahoo Set 2 & \textbf{74.15} & \textbf{93.56} \\
   \midrule 
   \midrule
   $\lambda$-MART & WEB10K & 48.22 & 81.85 \\
   $\lambda$-Loss & WEB10K & 48.33 & - \\
   $\mathbb{E}\lambda$-MART & WEB10K & 48.29 & 81.72. \\
   $\mathbb{E}\lambda$-Loss & WEB10K & 48.47 & - \\
   SoftRank       & WEB10K & 42.82 & 81.38 \\
   \midrule
   SR-$\mathcal{R}_1^{soft}$ & WEB10K & 48.19 & 83.08 \\
   SR-$\mathcal{R}_1$ & WEB10K & \textbf{48.53} & \textbf{83.30} \\
   \midrule
   \midrule
   $\lambda$-MART & WEB30K & 49.55 & 83.79 \\
   $\lambda$-Loss & WEB30K & 49.45 & - \\
   $\mathbb{E}\lambda$-MART & WEB30K & 49.49 & 83.79 \\
   $\mathbb{E}\lambda$-Loss & WEB30K & 49.52 & - \\
   SoftRank       & WEB30K & 43.46 & 82.73 \\
   \midrule
   SR-$\mathcal{R}_1^{soft}$ & WEB30K & \textbf{49.67} & \textbf{85.19} \\
   SR-$\mathcal{R}_1$ & WEB30K & 49.59 & 85.01 \\
   \midrule
    \bottomrule
  \end{tabular}\label{table:results} 
\end{table}

\begin{table}[t]
  \caption{Comparison of the algorithm's features on Yahoo Set 1, where $\pi_\mu$ means using unbiased smoothing.}
  \vspace{4pt}
  \centering
  \begin{tabular}{lc}
    \toprule
   Features & $\mathrm{NDCG@5}$ \\
   \midrule
   REINFORCE  & 70.74  \\
    CCS & 71.89 \\
    CCS+SFA &  74.55  \\
    CCS+SFA+SGLB (SR-$\mathcal{R}_1^{soft}$) & 74.68  \\
    CCS+SFA+SGLB+$\pi_\mu$ (SR-$\mathcal{R}_1$) &  \textbf{74.92}  \\
    \bottomrule
  \end{tabular}\label{table:results-yahoo} 
\end{table}

\paragraph{Parameter tuning} 
For all algorithms, we set the maximum number of trees to 1000. We tune the hyperparameters using 500 iterations of random search and select the best combination using the validation set, the details are given in Appendix~\ref{appendix:tuning}.

\paragraph{Results} 
The results are shown in Table~\ref{table:results}. One can see that StochasticRank (SR-$\mathcal{R}_1$) outperforms the baseline approaches on all datasets. In all cases, the difference with the closest baseline is statistically significant with a p-value $< 0.05$ measured by the paired one-tailed t-test. Also, in most cases, SR-$\mathcal{R}_1$ outperforms SR-$\mathcal{R}_1^{soft}$, which clearly demonstrates the advantage of unbiased smoothing, which takes into account the tie resolution policy.

The results in Table~\ref{table:results} are comparable to previously reported numbers, although they cannot be compared directly, since experimental setup (e.g., the maximum number of trees) is not fully described in many cases~\cite{lambdaloss}. More importantly, the previously reported results can be overvalued, since many openly available libraries compute ranking metrics using neither worst (as in our case) nor ``expected'' permutation, but some fixed arbitrary one depending on a particular implementation of the sorting operation.

To further understand how different techniques proposed in this paper affect the quality of the algorithm, we show the improvement obtained from each feature using the Yahoo dataset and the $\mathrm{NDCG}$ metrics (see Table~\ref{table:results-yahoo}). 
We see that CCS is significantly better than REINFORCE, while SFA gives an additional significant performance boost. SGLB and consistent smoothing further improve $\mathrm{NDCG}$. We note that for both REINFORCE and CCS we use one sample per gradient estimate since the most time-consuming operation for both estimates is sorting (see Appendix~\ref{algos}).

\section{Conclusion}\label{sec:conclusion}

In this paper, we proposed the first truly direct LTR algorithm. We formally proved that this algorithm converges \textit{globally} to the minimizer of the target loss function. This is possible due to the combination of three techniques: unbiased smoothing for consistency between the original and smoothed losses; SGLB for global optimization via gradient boosting; and CCS gradient estimate with uniformly bounded error and low variance, which is required for SGLB to be applied. Our experiments clearly illustrate that the new algorithm outperforms state-of-the-art LTR methods.

\bibliography{example_paper}

\begin{thebibliography}{33}
\providecommand{\natexlab}[1]{#1}
\providecommand{\url}[1]{\texttt{#1}}
\expandafter\ifx\csname urlstyle\endcsname\relax
  \providecommand{\doi}[1]{doi: #1}\else
  \providecommand{\doi}{doi: \begingroup \urlstyle{rm}\Url}\fi

\bibitem[Artstein \& Wets(1995)Artstein and Wets]{artstein}
Artstein, Z. and Wets, R.
\newblock Consistency of minimizers and the {SLLN} for stochastic programs.
\newblock \emph{Journal of Convex Analysis}, 2\penalty0 (1-2):\penalty0 1--17,
  1995.

\bibitem[{Bardet} et~al.(2015){Bardet}, {Gozlan}, {Malrieu}, and
  {Zitt}]{2015arXiv150702389B}
{Bardet}, J.-B., {Gozlan}, N., {Malrieu}, F., and {Zitt}, P.-A.
\newblock {Functional inequalities for Gaussian convolutions of compactly
  supported measures: explicit bounds and dimension dependence}.
\newblock \emph{arXiv e-prints}, art. arXiv:1507.02389, 2015.

\bibitem[Bruch et~al.(2020)Bruch, Han, Bendersky, and Najork]{48689}
Bruch, S., Han, S., Bendersky, M., and Najork, M.
\newblock A stochastic treatment of learning to rank scoring functions.
\newblock In \emph{Proceedings of the 13th ACM International Conference on Web
  Search and Data Mining (WSDM 2020)}, pp.\  61--69, 2020.

\bibitem[Burges(2010)]{burges2010ranknet}
Burges, C. J.~C.
\newblock From {RankNet} to {LambdaRank} to {LambdaMART}: An overview.
\newblock Technical report, Microsoft Research, 2010.

\bibitem[Cao et~al.(2006)Cao, Xu, Liu, Li, Huang, and Hon]{cao2006adapting}
Cao, Y., Xu, J., Liu, T.-Y., Li, H., Huang, Y., and Hon, H.-W.
\newblock Adapting ranking {SVM} to document retrieval.
\newblock In \emph{Proceedings of the 29th annual international ACM SIGIR
  conference on Research and development in information retrieval}, pp.\
  186--193. ACM, 2006.

\bibitem[CatBoost(2020)]{CatBoost_documentation}
CatBoost.
\newblock Ranking: objectives and metrics.
\newblock \url{https://catboost.ai/docs/concepts/loss-functions-ranking.html},
  2020.

\bibitem[Chapelle \& Chang(2011)Chapelle and Chang]{Chapelle2010}
Chapelle, O. and Chang, Y.
\newblock Yahoo! learning to rank challenge overview.
\newblock In \emph{Proceedings of the learning to rank challenge}, pp.\  1--24,
  2011.

\bibitem[Chen \& Guestrin(2016)Chen and Guestrin]{Chen:2016}
Chen, T. and Guestrin, C.
\newblock {XGBoost}: A scalable tree boosting system.
\newblock In \emph{Proceedings of the 22nd ACM SIGKDD International Conference
  on Knowledge Discovery and Data Mining}, pp.\  785--794, 2016.

\bibitem[Dolecki et~al.(1983)Dolecki, Salinetti, and Wets]{10.2307/1999442}
Dolecki, S., Salinetti, G., and Wets, R. J.-B.
\newblock Convergence of functions: equi-semicontinuity.
\newblock \emph{Transactions of the American Mathematical Society},
  276\penalty0 (1):\penalty0 409--429, 1983.

\bibitem[Erdogdu et~al.(2018)Erdogdu, Mackey, and
  Shamir]{Erdogdu:2018:GNO:3327546.3327636}
Erdogdu, M.~A., Mackey, L., and Shamir, O.
\newblock Global non-convex optimization with discretized diffusions.
\newblock In \emph{Proceedings of the 32Nd International Conference on Neural
  Information Processing Systems}, pp.\  9694--9703, 2018.

\bibitem[Ermoliev et~al.(1995)Ermoliev, Norkin, and Wets]{ErmolievNorkin}
Ermoliev, Y., Norkin, V., and Wets, R.
\newblock The minimization of semicontinuous functions: mollifier subgradients.
\newblock \emph{SIAM Journal on Control and Optimization}, 33, 01 1995.

\bibitem[Gelfand et~al.(1992)Gelfand, Doerschuk, and
  Nahhas-Mohandes]{GelfandGAA}
Gelfand, S.~B., Doerschuk, P.~C., and Nahhas-Mohandes, M.
\newblock Theory and application of annealing algorithms for continuous
  optimization.
\newblock In \emph{Proceedings of the 24th conference on Winter simulation},
  pp.\  494--499, 1992.

\bibitem[Ke et~al.(2017)Ke, Meng, Finley, Wang, Chen, Ma, Ye, and
  Liu]{LightGBM}
Ke, G., Meng, Q., Finley, T., Wang, T., Chen, W., Ma, W., Ye, Q., and Liu,
  T.-Y.
\newblock {LightGBM}: A highly efficient gradient boosting decision tree.
\newblock In \emph{Advances in neural information processing systems}, pp.\
  3146--3154, 2017.

\bibitem[Kustarev et~al.(2011)Kustarev, Ustinovsky, Logachev, Grechnikov,
  Segalovich, and Serdyukov]{kustarev2011smoothing}
Kustarev, A., Ustinovsky, Y., Logachev, Y., Grechnikov, E., Segalovich, I., and
  Serdyukov, P.
\newblock Smoothing {NDCG} metrics using tied scores.
\newblock In \emph{Proceedings of the 20th ACM international conference on
  Information and knowledge management}, pp.\  2053--2056, 2011.

\bibitem[Liu(2009)]{Liu:2009}
Liu, T.-Y.
\newblock Learning to rank for information retrieval.
\newblock \emph{Found. Trends Inf. Retr.}, 3\penalty0 (3):\penalty0 225--331,
  2009.

\bibitem[Nesterov \& Spokoiny(2017)Nesterov and Spokoiny]{Nesterov2017RandomGM}
Nesterov, Y. and Spokoiny, V.~G.
\newblock Random gradient-free minimization of convex functions.
\newblock \emph{Foundations of Computational Mathematics}, 17:\penalty0
  527--566, 2017.

\bibitem[Nguyen \& Sanner(2013)Nguyen and Sanner]{nguyen2013algorithms}
Nguyen, T. and Sanner, S.
\newblock Algorithms for direct 0--1 loss optimization in binary
  classification.
\newblock In \emph{International Conference on Machine Learning}, pp.\
  1085--1093, 2013.

\bibitem[Prokhorenkova et~al.(2018)Prokhorenkova, Gusev, Vorobev, Dorogush, and
  Gulin]{catboost}
Prokhorenkova, L., Gusev, G., Vorobev, A., Dorogush, A.~V., and Gulin, A.
\newblock Catboost: unbiased boosting with categorical features.
\newblock In \emph{Advances in Neural Information Processing Systems}, pp.\
  6638--6648, 2018.

\bibitem[Qin \& Liu(2013)Qin and Liu]{QinL13}
Qin, T. and Liu, T.
\newblock Introducing {LETOR} 4.0 datasets.
\newblock \emph{CoRR}, abs/1306.2597, 2013.

\bibitem[Qin et~al.(2010)Qin, Liu, and Li]{qin2010general}
Qin, T., Liu, T.-Y., and Li, H.
\newblock A general approximation framework for direct optimization of
  information retrieval measures.
\newblock \emph{Information retrieval}, 13\penalty0 (4):\penalty0 375--397,
  2010.

\bibitem[Raginsky et~al.(2017)Raginsky, Rakhlin, and
  Telgarsky]{DBLP:journals/corr/RaginskyRT17}
Raginsky, M., Rakhlin, A., and Telgarsky, M.
\newblock Non-convex learning via stochastic gradient langevin dynamics: a
  nonasymptotic analysis.
\newblock \emph{CoRR}, abs/1702.03849, 2017.

\bibitem[Rudin \& Wang(2018)Rudin and Wang]{rudin2018direct}
Rudin, C. and Wang, Y.
\newblock Direct learning to rank and rerank.
\newblock In \emph{International Conference on Artificial Intelligence and
  Statistics}, pp.\  775--783, 2018.

\bibitem[Sakai(2013)]{Sakai13}
Sakai, T.
\newblock Metrics, statistics, tests.
\newblock In \emph{Bridging Between Information Retrieval and Databases -
  {PROMISE} Winter School 2013, Bressanone, Italy, February 4-8, 2013. Revised
  Tutorial Lectures}, pp.\  116--163, 2013.

\bibitem[Tan et~al.(2013)Tan, Xia, Guo, and Wang]{tan2013direct}
Tan, M., Xia, T., Guo, L., and Wang, S.
\newblock Direct optimization of ranking measures for learning to rank models.
\newblock In \emph{Proceedings of the 19th ACM SIGKDD international conference
  on Knowledge discovery and data mining}, pp.\  856--864, 2013.

\bibitem[Taylor et~al.(2008)Taylor, Guiver, Robertson, and Minka]{Taylor2008}
Taylor, M., Guiver, J., Robertson, S., and Minka, T.
\newblock {SoftRank}: Optimizing non-smooth rank metrics.
\newblock In \emph{Proceedings of the 2008 International Conference on Web
  Search and Data Mining}, WSDM '08, pp.\  77--86, 2008.
\newblock ISBN 978-1-59593-927-2.

\bibitem[{Ustimenko} \& {Prokhorenkova}(2020){Ustimenko} and
  {Prokhorenkova}]{SGLB}
{Ustimenko}, A. and {Prokhorenkova}, L.
\newblock {SGLB: Stochastic Gradient Langevin Boosting}.
\newblock \emph{arXiv e-prints}, art. arXiv:2001.07248, 2020.

\bibitem[Volkovs \& Zemel(2009)Volkovs and Zemel]{volkovs2009boltzrank}
Volkovs, M.~N. and Zemel, R.~S.
\newblock {BoltzRank}: learning to maximize expected ranking gain.
\newblock In \emph{Proceedings of the 26th Annual International Conference on
  Machine Learning}, pp.\  1089--1096, 2009.

\bibitem[Vorobev et~al.(2019)Vorobev, Ustimenko, Gusev, and
  Serdyukov]{vorobev2019learning}
Vorobev, A., Ustimenko, A., Gusev, G., and Serdyukov, P.
\newblock Learning to select for a predefined ranking.
\newblock In \emph{International Conference on Machine Learning}, pp.\
  6477--6486, 2019.

\bibitem[Wang et~al.(2018)Wang, Li, Golbandi, Bendersky, and
  Najork]{lambdaloss}
Wang, X., Li, C., Golbandi, N., Bendersky, M., and Najork, M.
\newblock The lambdaloss framework for ranking metric optimization.
\newblock In \emph{Proceedings of The 27th ACM International Conference on
  Information and Knowledge Management (CIKM '18)}, pp.\  1313--1322, 2018.

\bibitem[Williams(1992)]{reinforce}
Williams, R.~J.
\newblock Simple statistical gradient-following algorithms for connectionist
  reinforcement learning.
\newblock \emph{Machine learning}, 8\penalty0 (3-4):\penalty0 229--256, 1992.

\bibitem[Wu et~al.(2010)Wu, Burges, Svore, and Gao]{wu2010adapting}
Wu, Q., Burges, C.~J., Svore, K.~M., and Gao, J.
\newblock Adapting boosting for information retrieval measures.
\newblock \emph{Information Retrieval}, 13\penalty0 (3):\penalty0 254--270,
  2010.

\bibitem[Xu et~al.(2010)Xu, Li, Liu, Peng, Lu, and Ma]{xudirect}
Xu, J., Li, H., Liu, T.-Y., Peng, Y., Lu, M., and Ma, W.-Y.
\newblock Direct optimization of evaluation measures in learning to rank.
\newblock 2010.

\bibitem[Yin et~al.(2016)Yin, Hu, Tang, Daly, Zhou, Ouyang, Chen, Kang, Deng,
  Nobata, Langlois, and Chang]{Yin2016}
Yin, D., Hu, Y., Tang, J., Daly, T., Zhou, M., Ouyang, H., Chen, J., Kang, C.,
  Deng, H., Nobata, C., Langlois, J.-M., and Chang, Y.
\newblock Ranking relevance in {Yahoo} search.
\newblock In \emph{Proceedings of the 22Nd ACM SIGKDD International Conference
  on Knowledge Discovery and Data Mining}, KDD '16, pp.\  323--332, 2016.

\end{thebibliography}
\bibliographystyle{icml2020}

\appendix
\section*{Appendix}

\begin{table}[t]
\caption{Notation.}
\label{tab:notation}
\begin{center}
\begin{small}
\begin{tabular}{cl}
\toprule
Variable & Description  \\
\midrule
$z \in \R^n$ & Vector of scores \\
$\xi \in \Xi_n$ & Vector of contexts \\
$r \in \R^n$ & Vector of relevance labels \\
$\theta \in \mathbb{R}^m$ & Vector of parameters\\
$L(z,\xi)$ & Loss function \\
$L_\xi^{\pi}(z, \sigma)$ & Smoothed loss function\\
$L_\xi^{\pi}(z, \sigma|z')$ & SFA smoothing of the loss\\
$\mathcal{L}(\theta)$ & Expected loss\\
$\mathcal{L}_N(\theta)$ & Empirical loss\\
$\mathcal{L}_N^{\pi}(\theta, \sigma)$ & Smoothed empirical loss \\
$\mathcal{L}_N^{\pi}(\theta, \sigma,\gamma)$ & Regularized and consistently smoothed loss\\
$\mathcal{R}_0$ & Scale-free discrete loss functions\\
$\mathcal{R}_1$ & Ranking loss functions \\
$\mathcal{R}_1^{soft}$ & Soft ranking loss functions \\
$\pi_{\xi}(z)$ & Distribution density for smoothing\\
$p_{\beta}(\theta)$ & Invariant measure of parameters \\
$p_{\beta}(F)$ & Invariant measure of predictions \\
$\sigma > 0$ & Smoothing standart deviation \\
$\beta > 0$ & Diffusion temperature \\
$\gamma > 0$ & Regularization parameter \\
$\mu \ge 0$ & Relevance shifting parameter \\
$\nu > 0$ & Scale-Free Acceleration parameter \\
\bottomrule
\end{tabular}
\end{small}
\end{center}
\vskip -0.1in
\end{table}

\section{Proof of Statement~\ref{statement-1}}\label{appendix:statement-1}

Let us prove that the set $\argmin_{\theta\in\mathbb{R}^m} \mathcal{L}_N(\theta)$ is not empty.

Consider $U_{ij}$ being open and convex sets for $V_i = \mathrm{im }\Phi_{\xi_i}$ (see Discreteness on subspaces in Definition~\ref{GRQF}). 
Then, $U_{ij}' = \Phi_{\xi_i}^{-1}U_{ij}\subset \mathbb{R}^m$ are also open and convex. 
Henceforth, the function $\mathcal{L}_N$ can be written as (ignoring the sets of zero measure):
\begin{equation}\mathcal{L}_N(\theta) = 
N^{-1}\sum_{j_1=1}^{k_1}\ldots \sum_{j_N=1}^{k_N}c_{j_1, \ldots j_N}\mathbb{1}_{\theta\in \cap_{i=1}^N U'_{i j_i}}\,.
\end{equation}

Henceforth, the function $\mathcal{L}_N$ is also discrete with open convex sets $\mathcal{U}_{s} := \cap_{i=1}^N U'_{i j_i}$ on the whole space $\mathbb{R}^m$. Hence, its $\argmin$ is one of these sets or their union.

\section{Stochastic smoothing}

\subsection{Mollification}\label{appendix:mollification}

A natural approach for smoothing is mollification~\cite{ErmolievNorkin, 10.2307/1999442}: choose a smooth enough distribution with p.d.f. $\pi(\theta)$, consider the family of distributions $\pi_{\delta}(\theta) = \delta^{-m}\pi(\delta^{-1}\theta)$, and let $\mathcal{L}_N(\theta, \delta) := \mathcal{L}_N\ast \pi_{\delta} \equiv \mathbb{E}_{\epsilon\sim \pi}\mathcal{L}_N(\theta + \delta\epsilon)$. Then, the minimizers of $\mathcal{L}_N(\theta, \delta)$ convergence to the minimizer of $\mathcal{L}_N(\theta)$.
Unfortunately, despite theoretical soundness, it is hard to derive efficient gradient estimates even in the linear case $f_{\xi_i}(\theta) = \Phi_{\xi_i}\theta$. Moreover, in the gradient boosting setting, we do not have access to all possible coordinates of $\theta$ at each iteration. Henceforth, we cannot use the mollification approach directly.

Thus, instead of acting on the level of parameters $\theta$, we act on the level of scores $z$: $L_{\xi}^{\pi}(z, \sigma):=\mathbb{E}L(z+\sigma\varepsilon, \xi)$, where $\varepsilon$ has p.d.f. $\pi(z)$. We multiply the noise by $\sigma$ to preserve Scalar-freeness in a sense that $L^{\pi}_\xi(\lambda z, \lambda \sigma) = L_\xi^{\pi}(z, \sigma)$ for any $\lambda > 0$.

In the linear case $f(\theta) = \Phi\theta$, if $\mathrm{rk}\Phi = n$, it is not hard to show the convergence of minimizers. Indeed, we can obtain mollification by ``bypassing'' the noise from scores to parameters by multiplying on $\Phi^{-1}$. However, in general, we cannot assume $\mathrm{rk}\Phi = n$.

\subsection{Proof of Theorem~\ref{theorem-1}}\label{appendix:theorem-1}

The trick is to proceed with $L(f_{\xi_i}(\theta), \xi_i)$ and to show that there exists an open and dense set $U_{\xi_i}\subset \mathbb{R}^{m}$ such that the convergence is locally uniform as $\sigma \rightarrow 0_+$, $\mu \rightarrow \infty$, $\sigma \mu \rightarrow 0_+$.

Let us proceed with proving the existence of such $U_{\xi_i}\forall i$. Let us define 
\begin{multline*}
U_{\xi_i}:=\Big\{\theta\in \mathbb{R}^m:\forall j\ne j' \big(f_{\xi_i}(\theta)_j = f_{\xi_i}(\theta)_{j'}\big)\Rightarrow\\ 
    \forall \theta'\in\mathbb{R}^m\big(f_{\xi_i}(\theta')_j = f_{\xi_i}(\theta')_{j'}\big) \Big\}.
\end{multline*}

Clearly, the set is not empty, open, and dense. Now, take an arbitrary $\theta \in U_{\xi_i}$. Consider $z = f_{\xi_i}(\theta)$ and divide the set $\{1, \ldots, n_i\}$ into disjoint subsets $J_1, \ldots, J_k$ such that all components $z_j$ corresponding to one group are equal and all components $z_j$ corresponding to different $J$'s are different. Clearly, we need to ``resolve'' only those which are equal: for small enough $\sigma \approx 0, \sigma\mu \approx 0$ we obtain that even after adding the noise $f_{\xi_i}(\theta')-\sigma\mu r + \sigma\varepsilon$ the order of $J$'s is preserved with high probability uniformly in some vicinity of $\theta$, whilst for large enough $\mu \gg 1$ we obtain the worst case permutation of $z_j$ corresponding to the one group with high probability uniformly on the whole $U_{\xi_i}$. Thus, we obtain locally uniform convergence $\mathbb{E}L(f_{\xi_i}(\theta) - \sigma \mu r + \sigma \varepsilon, \xi_i) \rightarrow L(f_{\xi_i}(\theta), \xi_i)$.

\subsection{Proof of Theorem~\ref{theorem-2}}\label{appendix:theorem-2}

Clearly, the conditions of the theorem imply that for general $\theta$ w.l.o.g. we can assume that $\Phi_{\xi_i}\theta \in U_{ij_i}$ for some indexes $j_i$. Henceforth, after adding the noise with $\sigma \rightarrow 0_+$ we must obtain locally uniform approximation since the functions $L(z, \xi_i)$ are locally constant in a vicinity of $z = \Phi_{\xi_i}\theta \, \forall i$.

\subsection{Consistent smoothing for LSO}\label{appendix:smoothing-lso}

\begin{theorem}
In gradient boosting, if $L(\cdot, \cdot) \in \mathcal{R}_0$ is coming from the LSO problem, then any smoothing is consistent.
\end{theorem}

\begin{proof} Conditions from Theorem~\ref{theorem-2} translate into a condition that $\big(\Phi_{\xi}\theta\big)_j \ne 0$ for all $j$ and for all $\theta$ almost surely. This can be enforced by adding a free constant to the linear model, but in the gradient boosting setting this condition is essentially satisfied: consider $\theta = \mathbb{1}_m$, then $\big(\Phi_{\xi}\mathbb{1}_m\big)_j \ge 1\,\forall j$ since the matrix $\Phi_{\xi}$ is $0$-$1$ matrix and have at least one ``$1$'' in each row (every item fells to at least one leaf of each tree). Henceforth, for any general $\theta$ we can assume another general $\widetilde{\theta} = \theta + \nu\mathbb{1}_m$, where $\nu$ is any random variable with absolute continuous p.d.f. This in turn implies $\big(\Phi_{\xi}\widetilde{\theta}\big)_j \ne 0$ almost surely. Henceforth, Theorem~2 holds ensuring the consistency of smoothing.
\end{proof}

\section{Coordinate Conditional Sampling}\label{appendix:CCS}

\subsection{Proof of Lemma~\ref{lemma=finite-jumps}}

Consider a line $H = \{(z_j, z_{\backslash j}):\forall z_j\in\mathbb{R}\}$ and subsets $U_1,\cdots, U_k$ for $k = k(n, \mathbb{R}^n)$ from the Discretness on subspaces assumption for $V = \mathbb{R}^n$. Then $U_i \cap H = (a_i, b_i)\times \{z_{\backslash_j}\}$ due to opennes and convexity of $U_i$ for $a_i, b_i\in \mathbb{R}\cup \{\pm\infty\}$. Moreover, $(U_i \cap H) \cap (U_{i'} \cap H) = \emptyset \, \forall i\ne i'$ and, by ignoring sets of zero measure, we can assume that $\overline{\cup_i (a_i, b_i)\times \{z_{\backslash j}\}} = H$. After that, we can take all finite $\{b_1, \ldots, b_k\}\cap \mathbb{R}$ as breaking points.

\subsection{Proof of Theorem~\ref{theorem-3}}

Observe that $L\ast \pi^{j}_\xi$ tautologically equals $l_j\ast \pi^{j}_\xi$ and the convolution is distributive with respect to summation, so we can write:
$$L\ast \pi^{j} = \sum_{s=1}^{k'} \Delta l_j(b_s) \mathbb{1}_{\{z_j \le b_s\}}\ast \pi^{j}_\xi +  \mathrm{const}(z_{\backslash j}).$$
The convolution $\mathbb{1}_{\{z_j \le b_s\}}\ast \pi^{j}_\xi$ is equal to $\mathbb{P}_\xi(z_j + \sigma\varepsilon_j < b_s|\varepsilon_{\backslash j}) := \sigma^{-1}\int_{\mathbb{R}} \mathbb{1}_{\{z_j + \sigma\varepsilon_j \le b_s\}}\pi^{j}_\xi(\sigma^{-1}\varepsilon_j)\mathrm{d}\varepsilon_j$, allowing us to rewrite:
\begin{multline*}
L\ast \pi^{j}_\xi \\ = \sum_{s=1}^{k'} \Delta l_j(b_s) \mathbb{P}_\xi(\varepsilon_j <\sigma^{-1}(b_s - z_j)|\varepsilon_{\backslash j}) + \mathrm{const}(z_{\backslash j})\,.
\end{multline*}
The above formula is ready for differentiation since each term is actually a $C^{(2)}(\mathbb{R})$ function by the variable $z_j$:
$$\frac{\partial}{\partial z_j}L\ast \pi^{j}_\xi = -\sigma^{-1}\sum_{s=1}^{k'} \Delta l_j(b_s) \pi^{j}(\sigma^{-1}(b_s - z_j)).$$
After the convolution with $\pi_\xi^{\backslash j}$, we finally get the required formula.

\subsection{Proof of Corollary~\ref{cor-1}}

For LTR ($\mathcal{R}_1$ and $\mathcal{R}_1^{soft}$), all these $b_s$ actually lay in $\{z_1, \ldots, z_n\}\subset \mathbb{R}$ due to Pairwise decision boundary assumption and, henceforth, we do not need to compute them, we just need to take coordinates of $z\in\mathbb{R}^{n}$ as breaking points and note that if some of $z_s$ is not a breaking point for $L(z, \xi)$, then essentially $\Delta l_j(z_s) = 0$. Then, we can write
$$\frac{\partial}{\partial z_j}L\ast \pi^{j}_\xi = -\sigma^{-1}\sum_{s=1}^{n} \Delta l_j(z_s) \pi^{j}_\xi(\sigma^{-1}(z_s - z_j)).$$

Let us note that for LSO, we can actually take $k' = 1$ and $b_1 = 0$ and simplify the formula to:
$$l_j(z_j) = \Delta l_j \mathbb{1}_{\{z_j \le 0\}} + \mathrm{const}(z_{\backslash j}).$$

\subsection{Proof of Theorem~\ref{theorem-4}}

\begin{lemma}
\label{lemma:sigma-der}
The function $L_\xi^{\pi}(z, \sigma)$ satisfies the following linear first order Partial Differential Equation (PDE):
$$\frac{\partial}{\partial \sigma}L_\xi^{\pi}(z, \sigma) = - \sigma^{-1}\langle\nabla_{z}L^{\pi}_\xi(z, \sigma), z\rangle_2.$$
\end{lemma}
\begin{proof}
The proof is a direct consequence of Scalar-Freenees: we just need to differentiate the equality $L^{\pi}_\xi(\alpha z, \alpha\sigma)\equiv L^{\pi}_\xi(z, \sigma)$ (holding for $\alpha > 0$) by $\alpha$ and set $\alpha = 1$.
\end{proof}
\begin{lemma}
$\frac{\partial}{\partial \sigma}L_\xi^{\pi}(z, \sigma)$ is uniformly bounded by $\mathcal{O}(\sigma^{-1})$.
\end{lemma}
\begin{proof}
Consider writing $L_\xi^{\pi}(z, \sigma)$ in the integral form:
$$L_\xi^{\pi}(z, \sigma) = \sigma^{-n}\int_{\mathbb{R}^n} L(z + \varepsilon, \xi)\pi(\sigma^{-1}\varepsilon)\mathrm{d}\varepsilon.$$
By Fubini's theorem, we can pass the differentiation $\frac{\partial}{\partial \sigma}$ to inside the integral and obtain:
\begin{multline*}
\frac{\partial}{\partial \sigma}L_\xi^{\pi}(z, \sigma) = -n\sigma^{-n-1}\int_{\mathbb{R}^n} L(z + \varepsilon, \xi)\pi(\sigma^{-1}\varepsilon)\mathrm{d}\varepsilon \\
- \sigma^{-n-2}\int_{\mathbb{R}^n} L(z + \varepsilon, \xi)\langle\nabla\pi(\sigma^{-1}\varepsilon), \varepsilon\rangle\mathrm{d}\varepsilon.
\end{multline*}
Consider the variable $\varepsilon' = \sigma^{-1}\varepsilon$, then we arrive at
\begin{multline*}
\frac{\partial}{\partial \sigma}L_\xi^{\pi}(z, \sigma) = -n\sigma^{-1}\int_{\mathbb{R}^n} L(z + \sigma\varepsilon, \xi)\pi(\varepsilon)\mathrm{d}\varepsilon \\
- \sigma^{-1}\int_{\mathbb{R}^n} L(z + \sigma\varepsilon, \xi)\langle\nabla\pi(\varepsilon), \varepsilon\rangle\mathrm{d}\varepsilon.
\end{multline*}
Taking the absolute value of both sides and using the triangle inequality, we derive
\begin{multline*}
\Big|\frac{\partial}{\partial \sigma}L_\xi^{\pi}\Big| \le  nl\sigma^{-1} + l\sigma^{-1}\int_{\mathbb{R}^n}\|\nabla\pi(\varepsilon)\|_2 \|\varepsilon\|_2\mathrm{d}\varepsilon,
\end{multline*}
where $l = \sup_z |L(z, \xi)| < \infty$ by the Uniform boundedness assumption and the last integral is well defined by the Derivative decay assumption.
\end{proof}
\begin{corollary}
$\sup_z \Big|\big\langle \nabla_z L_\xi^{\pi}, z\big\rangle_2\Big| = \mathcal{O}(1)$ independently from $\sigma$.
\end{corollary}
\begin{proof}
Immediate consequence of the previous lemmas.
\end{proof}

Now, assume that $\sigma = \sigma(z)$ is differentiable and non-zero at $z$. The following lemma describes $\nabla_z L_\xi^{\pi}(z, \sigma(z))$ in terms of $\nabla_z L_\xi^{\pi} := \nabla_z L_\xi^{\pi}(z, \sigma)\big|_{\sigma=\sigma(z)}$.

\begin{lemma}
\label{scalar-free-differential-equiation}
The following formula holds:
$$\nabla_z L_\xi^{\pi}(z, \sigma(z)) = \nabla_z L_\xi^{\pi} -\big\langle \nabla_z L_\xi^{\pi}, z\big\rangle_2 \nabla_z \log \sigma(z).$$
\end{lemma}
\begin{proof}
Consider writing 
$$\nabla_z L_\xi^{\pi}(z, \sigma(z)) = \nabla_z L_\xi^{\pi} + \frac{\partial}{\partial \sigma}L_\xi^{\pi}(z, \sigma(z)) \nabla_z \sigma(z).$$
Then, by Lemma \ref{lemma:sigma-der} we obtain the formula.
\end{proof}

\section{Fast ranking metrics computation}
\label{algos}

We need to be able to compute $L(z', z_{\backslash s_i} + \sigma\varepsilon_{\backslash s_i}, \xi)$ for an arbitrary $z'\in\mathbb{R}$ and a position $i$, where $s \in S_n$ represents $s := \mathrm{argsort}(z+\sigma\varepsilon)$ for the CCS estimate (note that there is no ambiguity in computing $\mathrm{argsort}$ since with probability one $z_{j_1}+\sigma\varepsilon_{j_1} \ne z_{j_2} + \sigma\varepsilon_{j_2}$ for $j_1 \ne j_2$). Moreover, $\mathrm{argsort}$ requires  $\mathcal{O}(n\log n)$ operations.

Typically, the evaluation of $L(\cdots)$ costs $\mathcal{O}(n)$, e.g., for $\mathrm{ERR}$. Fortunately, for many losses it is possible to exploit the structure of the loss that allows evaluating $L$ in $\mathcal{O}(1)$ operations using some precomputed shared cumulative statistics related to the loss which can be computed in $\mathcal{O}(n)$ operations and $\mathcal{O}(n)$ memory. 

For all $L \in \mathcal{R}_1$ in the worst case we need $\mathcal{O}(n^2)$ evaluations of $L$ to compute the CCS (for each of $n$ coordinates to sum up at most $n$ evaluations). Thus, the overall worst case asymptotic of the algorithm would be $\mathcal{O}(n\log n + n + n^2) = \mathcal{O}(n^2)$ if the evaluation costs $\mathcal{O}(1)$. For the sake of simplicity, we generalize both $\mathrm{NDCG@k}$ and $\mathrm{ERR}$ into one class of losses:
\begin{equation}
    \label{GMC}
    L(z, \xi) = -\sum_{i=1}^{n} w_i g(r_{s_i}) \prod_{j=1}^{i-1} d_{s_j},
\end{equation}
where $W = \{w_i\}_{i=1}^n$ are some predefined positions' weights typically picked as $\frac{\mathbb{1}_{\{i \le k\}}}{\max_z \mathrm{DCG@k} \log (i + 1)}$ for $\mathrm{NDCG@k}$ and $\frac{1}{i}$ for $\mathrm{ERR}$); $D = \{d_i\}_{i=1}^n$ is typically picked as $d_i = 1\forall i$ for $-\mathrm{NDCG@k}$ and $d_i = 1 - r_i \, \forall i$ for $\mathrm{ERR}$; and finally we define $g(r) = r$ for $r \in [0, 1]$ and $g(r) = \frac{2^r - 1}{2^4}$ for $r \in \{0, 1, 2, 3, 4\}$. 

First, we need to define and compute the following cumulative product:
\begin{equation*}
\begin{split}
    p_m &= d_{s_{m - 1}} p_{m-1} = \prod_{j=1}^{m-1}d_{s_j} \text{ if }m > 1,
\end{split}
\end{equation*}
where $p_1 = 1$. Denote $P := \{p_i\}_{i=1}^n$. Next, we use them we define the following cumulative sums:
\begin{equation*}
\begin{split}
    S^{\mathrm{up}}_m &= S^{\mathrm{up}}_{m-1} + w_{m + 1} g(r_{s_m}) p_m  \text{ if }m > 1,
\end{split}
\end{equation*}
\begin{equation*}
\begin{split}
    S^{\mathrm{mid}}_m &= S^{\mathrm{mid}}_{m-1} + w_{m} g(r_{s_m}) p_m \text{ if }m > 0,
\end{split}
\end{equation*}
\begin{equation*}
\begin{split}
    S^{\mathrm{low}}_m &= S^{\mathrm{low}}_{m-1} + w_{m - 1} g(r_{s_m}) p_m  \text{ if }m > 0,
\end{split}
\end{equation*}
where $S^{\mathrm{up}}_0 = S^{\mathrm{up}}_1 = S^{\mathrm{mid}}_0 = S^{\mathrm{low}}_0 = 0$.

All these cumulative statistics can be computed at the same time while we compute $L(z + \sigma\varepsilon, \xi)$. Note that we need additional $O(n)$ memory to store these statistics. 

Now fix a position $i$ and score $z'$. Express $L(z', z_{\backslash s_i} + \sigma\varepsilon_{\backslash s_i}, \xi)$ as $(L(z', z_{\backslash s_i} + \sigma\varepsilon_{\backslash s_i}, \xi) - L(z + \sigma\varepsilon, \xi)) + L(z + \sigma\varepsilon, \xi)$. Thus, we need to compute $L(z', z_{\backslash s_i} + \sigma\varepsilon_{\backslash s_i}, \xi) - L(z + \sigma\varepsilon, \xi)$.

If $z' > z_{s_i} + \sigma \varepsilon_{s_i}$, we define $i' := i$; otherwise, define $i' := i-1$ --- this variable represents the new position of the $s_i$-th document in $z+\sigma \varepsilon$. Also, if $z' > z_{s_i} + \sigma \varepsilon_{s_i}$, we define: 
\begin{equation*}
    \begin{split}
        & T^\mathrm{{low}} = S^{\mathrm{mid}}_{i'} - S^{\mathrm{mid}}_{i},\\
        & T^{\mathrm{up}} = d_{s_i}^{-1}(S^{\mathrm{up}}_{i'} - S^{\mathrm{up}}_{i}),\\
        & w  = w_{i} p_i,\\
        & w' = w_{i'} d_{s_i}^{-1}p_{i'}.
    \end{split}
\end{equation*}
Otherwise, define: 
\begin{equation*}
    \begin{split}
        & T^{\mathrm{low}} = d_{s_i}(S^{\mathrm{low}}_{i'} - S^{\mathrm{low}}_{i - 1}),\\
        & T^{\mathrm{up}} = S^{\mathrm{mid}}_{i'} - S^{\mathrm{mid}}_{i-1},\\
        & w  = w_{i} p_i,\\
        & w' = w_{i' - 1} p_{i'}.
    \end{split}
\end{equation*}

Then, we calculate $L(z', z_{\backslash s_i} + \sigma\varepsilon_{\backslash s_i}, \xi) - L(z + \sigma\varepsilon, \xi)$ as $g(r_{s_i}) (w - w') - (T^{\mathrm{up}} - T^{\mathrm{low}})$. The meaning of the formula is simple: we measure the change of gain of the $s_i$-th document if we change its score to $z'$ from $z_{s_i}+\sigma\varepsilon_{s_i}$ minus the difference of gains of all documents on positions from $i'$ up to $i - 1$, if $i' < i$, and from $i + 1$ up to $i' - 1$, if $i' > i$.

The above formulas can be verified directly by evaluating the cases when $z' > z_{s_i} + \sigma \varepsilon_{s_i}$ or $z' < z_{s_i} + \sigma \varepsilon_{s_i}$ and expanding $S^{*}_m$ as $\sum_i w_{i \pm 1} g(r_{s_i}) p_i$. Note that all differences $S^{*}_{i} - S^{*}_{j}$ take into account all documents on positions from $j + 1$ up to $i$ inclusively.

Note that $S^{\mathrm{mid}}_{n} \equiv L(z + \sigma\varepsilon, \xi)$. Indeed, $$\sum_{i=1}^{n} w_i g(r_{s_i}) p_i = \sum_{i=1}^{n} w_i g(r_{s_i}) \prod_{j=1}^{i-1} d_{s_j} = L(z + \sigma\varepsilon, \xi).$$
Therefore, we obtain:
\begin{multline}
    L(z', z_{\backslash s_i} + \sigma\varepsilon_{\backslash s_i}, \xi) = g(r_{s_i}) (w - w') \\ - (T^{\mathrm{up}} - T^{\mathrm{low}}) + S^{\mathrm{mid}}_k.
\end{multline}

\section{Global Optimization by Diffusion}

\subsection{Overview of SGLB idea}\label{appendix:sglb}

Global convergence of SGLB is guaranteed by a so-called Predictions' Space Langevin Dynamics Stochastic Differential Equation
\begin{equation*}
\begin{split}
\mathrm{d}F(t) = -\gamma F(t)\mathrm{d}t - P\nabla_F \mathcal{L}_N^{\pi}&(F(t), \sigma)\mathrm{d}t \\
&+ \sqrt{2\beta^{-1}P}\mathrm{d}W(t),
\end{split}
\end{equation*}
where $F(t) := \Phi\theta(t) = (\Phi_{\xi_1}\theta(t), \ldots, \Phi_{\xi_N}\theta(t)) = (f_{\xi_1}(\theta), \ldots, f_{\xi_N}(\theta)) \in \mathbb{R}^{N'}$ denotes the predictions Markov Process on the train set $\mathcal{D}_N$,  $W(t)$ is a standard Wiener process with values in $\mathbb{R}^{N'}$, $N' := \sum_{i=1}^N n_i$, $P = P^T$ is an implicit preconditioner matrix of the boosting algorithm, and $\beta > 0$ is a temperature parameter that controls exploration/exploitation trade-off. Note that here we override the notation $\mathcal{L}_N(F) \equiv \mathcal{L}_N(\theta)$ since $F = \Phi \theta$. Further by $\Gamma = \sqrt{P^{-1}}$ we denote an implicitly defined regularization matrix.

The global convergence is implied by the fact that as $t \rightarrow \infty$, the stationary distribution $p_\beta(F)$ of $F(t)$ concentrates around the global optima of the implicitly regularized loss $$\mathcal{L}^{\pi}_N(F, \sigma, \gamma) = \mathcal{L}^{\pi}_N(F, \sigma) + \frac{\gamma}{2} \|\Gamma F\|_2^2\,.$$ 
More formally, the stationary distribution is $p_\beta(F) \propto \exp(-\beta \mathcal{L}^{\pi}_N(F, \sigma, \gamma))$. According to \citet{SGLB}, optimization is performed within a linear space $V := \mathrm{im\,}\Phi$ that encodes all possible predictions $F$ of all possible ensembles formed by the weak learners associated with the boosting algorithm. We refer interested readers to \cite{SGLB} for the details. 

\subsection{Proof of Theorem~\ref{theorem-5}}\label{appendix:theorem-5}

Let us first prove the following lemma.

\begin{lemma}
\label{thm:lipschitz}
The function $\mathcal{L}_N^{\pi}(F, \sigma)$ is uniformly bounded, Lipschitz continuous with constant $L_0 = \mathcal{O}(\sigma^{-1})$, and Lipschitz smooth with constant $L_1 = \mathcal{O}(\sigma^{-2})$.
\end{lemma}
\begin{proof} The proof of Lipschitz continuity is a direct consequence of the uniform boundedness by $\mathcal{O}(\sigma^{-1})$ of CCS. If we differentiate CCS estimate one more time, we obtain the estimates for the Hessian that must be uniformly bounded by $\mathcal{O}(\sigma^{-2})$ due to the uniform boundedness of $\nabla \pi$, thus giving Lipschitz smoothness.
\end{proof}

In addition to Lipschitz smoothness, continuity, and boundedness from above, we also need $\|\widehat{\nabla}_{CC}\mathcal{L}_N^{\pi}(F, \sigma) - \nabla\mathcal{L}_N^{\pi}(F, \sigma)\|_2 = \mathcal{O}(1)$~\cite{SGLB}, but that condition is satisfied since both terms are uniformly bounded by $\mathcal{O}(\sigma^{-1})$. Thus, the algorithm has limiting stationary measure $p_\beta(F) \propto \exp(-\beta \mathcal{L}^{\pi}_N(F, \sigma, \gamma))$. 

Then, consistency of the smoothing ensures that as $\sigma \rightarrow 0_+$, $p_\beta(F) \rightarrow p_\beta^*(F)$, where $p_\beta^*(F) \propto \exp(-\beta(\mathcal{L}_N(F) + \frac{\gamma}{2}\|\Gamma F\|_2^2))$  and thus for $\beta \gg 1$ the measures $p_\beta^*$ and $p_\beta$ for $\sigma \approx 0$ concentrate around the global optima of $\mathcal{L}_N(F)$.

\subsection{Proof of Theorem~\ref{theorem=gen-gap}}\label{appendix:theorem-6}

Following \citet{DBLP:journals/corr/RaginskyRT17, SGLB}, we immediately obtain that $\big|\mathbb{E}_{\theta\sim p_\beta(\theta)} \mathcal{L}^{\pi}(\theta, \sigma) - \mathbb{E}_{\theta\sim p_\beta(\theta)}\mathcal{L}_N^{\pi}(\theta, \sigma)\big| = \mathcal{O}(\frac{(\beta+d)^2}{N\lambda_*})$ with $\lambda_* > 0$ and $d = V_{\mathcal{B}}$. In general non-convex case $\frac{1}{\lambda_*}$ can be of order $\exp(\mathcal{O}(d))$ \cite{DBLP:journals/corr/RaginskyRT17} but for smoothed SF losses we can give a better estimate without exponential dependence on the dimension.

Observe that our measure is the sum of uniformly bounded Lipschitz smooth with constant $\mathcal{O}(\sigma^{-2})$ and a Gaussian $\frac{\gamma}{2}\|\Gamma\Phi\theta\|_2^2$, then the more appropriate bound from the logarithmic Sobolev inequality applies according to Lemma~2.1~\cite{2015arXiv150702389B} $\frac{1}{\lambda_*} = \mathcal{O}\left(\frac{\exp(\mathcal{O}(\frac{\beta}{\gamma\sigma^2}))}{\gamma\beta}\right)$ being dimension-free. Note that Miclo's trick in the proof of the lemma should be skipped since $\mathcal{L}^{\pi}_N(\theta, \sigma)$ is already fine enough. Coupling the spectral gap bound with the generalization gap, we obtain the theorem.

\section{Parameter tuning}\label{appendix:tuning}

For tuning, we use the random search (500 samples) with the following distributions:
\begin{itemize}
    \item For \textit{learning-rate} log-uniform distribution over $[10^{-3}, 1]$.
    \item For \textit{l2-leaf-reg} log-uniform distribution over $[10^{-1}, 10^1]$ for baselines and \textit{l2-leaf-reg=0} for StochasticRank.
    \item For noise strength \cite{48689} uniform distribution over $[0, 1]$.
    \item For \textit{depth} uniform distribution over $\{6, 7, 8, 9, 10\}$.
    \item For \textit{model-shrink-rate} log-uniform distribution over $[10^{-5}, 10^{-2}]$ for StochasticRank.
    \item For \textit{diffusion-temperature} log-uniform distribution over $[10^8, 10^{11}]$ for StochasticRank.
    \item For \textit{mu} log-uniform distribution over $[10^{-2}, 10]$ for StochasticRank-$\mathcal{R}_1$.
\end{itemize}

\end{document}